\documentclass[pmlr]{jmlr}

\usepackage{longtable}

\usepackage{booktabs}
\usepackage[load-configurations=version-1]{siunitx} 

\usepackage{times}
\usepackage{graphicx}
\usepackage{amsmath}
\usepackage{amssymb}

\usepackage{caption}
\usepackage{multirow}

\DeclareMathOperator*{\argmin}{argmin}

\theorembodyfont{\upshape}
\theoremheaderfont{\scshape}
\theorempostheader{:}
\theoremsep{\newline}

\jmlrvolume{1}
\jmlryear{2021}
\jmlrworkshop{NeurIPS 2020 Competition and Demonstration Track}

\title[MosAIc: Finding Artistic Connections across Culture with Conditional Image Retrieval
]{MosAIc: Finding Artistic Connections across Culture with Conditional Image Retrieval}

  \author{\Name{Mark Hamilton}$^{1,2}$, 
  \Name{Stephanie Fu}$^2$, 
  \Name{Mindren Lu}$^2$, 
  \Name{Johnny Bui}$^2$,
  \Name{Darius Bopp}$^2$, 
  \Name{Zhenbang Chen}$^2$, 
  \Name{Felix Tran}$^2$, 
  \Name{Margaret Wang}$^2$, 
  \Name{Marina Rogers}$^2$,  
  \Name{Lei Zhang}$^1$, 
  \Name{Chris Hoder}$^1$,
  \Name{William T. Freeman} $^{2,3}$\\
  \addr $^1$Microsoft,
  \addr $^2$MIT,
  \addr $^3$Google
  }

\begin{document}

\maketitle
\begin{figure}[h]
    \centering
    \vspace{-.3in}
    \includegraphics[width=.9\textwidth]{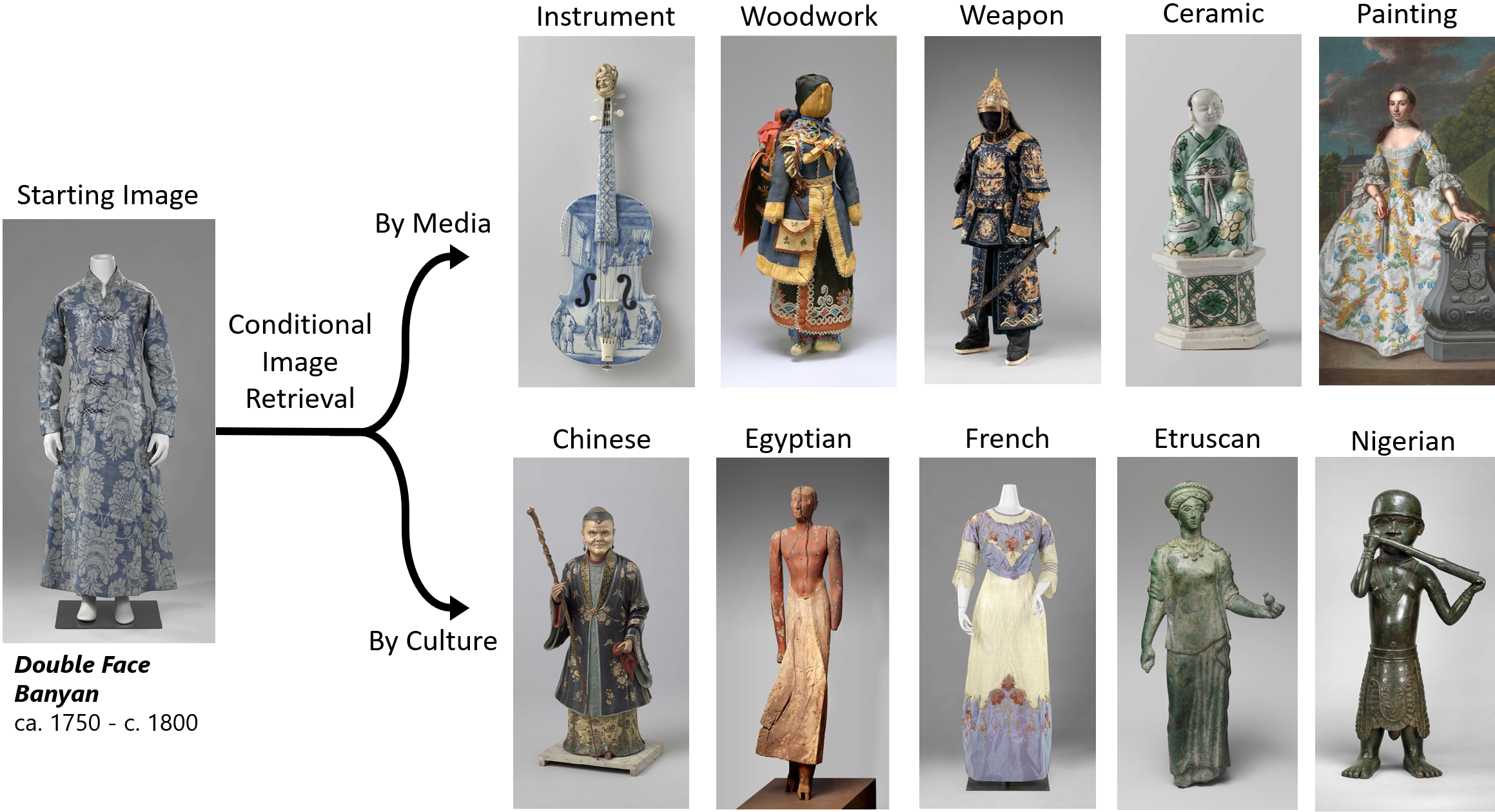}
    \caption{Conditional image retrieval results on artwork from the Metropolitan Museum of Art and Rijksmuseum using media and culture (text above images) as conditioners.}
    \label{fig:feature-graphic}
\end{figure}%

\begin{abstract}

We introduce MosAIc, an interactive web app that allows users to find pairs of semantically related artworks that span different cultures, media, and millennia. To create this application, we introduce Conditional Image Retrieval (CIR) which combines visual similarity search with user supplied filters or ``conditions''. This technique allows one to find pairs of similar images that span distinct subsets of the image corpus. We provide a generic way to adapt existing image retrieval data-structures to this new domain and provide theoretical bounds on our approach's efficiency. To quantify the performance of CIR systems, we introduce new datasets for evaluating CIR methods and show that CIR performs non-parametric style transfer. Finally, we demonstrate that our CIR data-structures can identify ``blind spots'' in Generative Adversarial Networks (GAN) where they fail to properly model the true data distribution.

\end{abstract}

\section{Introduction}

In many Image Retrieval (IR) applications, it is natural to limit the scope of the query to a subset of images. For example, returning similar clothes by a certain brand, or similar artwork from a specific artist. Currently, it is a challenge for IR systems to restrict their attention to sub-collections of images on the fly, especially if the subset is very distinct from the query image. This work explores how to create image retrieval systems that work in this setting, which we call ``Conditional Image Retrieval" (CIR). We find that CIR can uncover pairs of artworks within the combined open-access collections of the Metropolitan Museum of Art \citep{Met} and the Rijksmuseum \citep{Rijks} that have striking visual and semantic similarities despite originating from vastly different cultures and millennia and introduce an interactive web app MosAIc (\url{www.aka.ms/mosaic}) to demonstrate the approach. To understand our methods better, we evaluate CIR on the FEI Face Database \citep{thomaz2010new} as well as two new large-scale image datasets that we introduce to help evaluate these systems. These experiments show that CIR can perform a non-parametric variant of ``style transfer'' where neighbors in different subsets have similar content but are in the ``style'' of the target subset of images.

\begin{figure}[tbp]
\centering
\begin{minipage}{.49\textwidth}
  \centering
  
  \includegraphics[width=\linewidth]{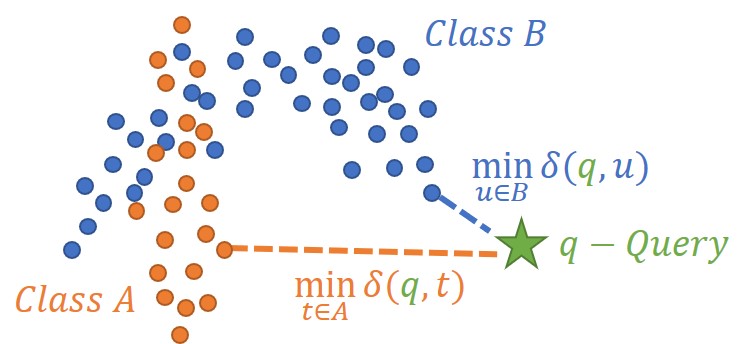}
  \captionof{figure}{Conditional K-Nearest Neighbors for a query point, $q$, and distance, $\delta$, on a simple two class dataset.}
    \label{fig:example}
\end{minipage}%
\hfill
\begin{minipage}{.49\textwidth}
  \centering
        \begin{tabular}{c|cc}
Component         & Space Efficiency                    & Measured \\ \hline\hline
Data              & $\mathcal{O}(n \times d)$           & 16 GB    \\
Tree              & $\mathcal{O}((2n / l) \times d )$   & 65 MB    \\
Cond. Index       & $\mathcal{O}(c \times 2n / l)$      & 6.4 MB   \\ \hline
\end{tabular}
\captionof{table}{Space efficiency of a binary CKNN Tree with number of points, $n$, dimensionality, $d$, leaf size $l$, and number of classes in the index, $c$. Measured results are from a tree built on the Conditional Art dataset: $n=1000000, d=2048, l=500, c=200$.}
    \label{table:mem}
\end{minipage}
\end{figure}

 We also investigate ways to improve IR system performance in the conditional setting. One challenge current systems face is that a core component of many IR systems, K-Nearest Neighbor (KNN) data-structures, only support queries over the entire corpus. Restricting retrieved images to a particular class or filter requires filtering the ``unconditional'' query results, switching to brute force adaptively \citep{matsui2018reconfigurable}, or building a new KNN data-structure for each filter. The first approach is used in several production image search systems \citep{degenova2017, bing_2017, mellina_2017}, but can be costly if the filter is specific, or the query image is far from valid images. Switching to brute force adaptively can mitigate this problem but is limited by the speed of brute force search, and its performance will degrade if the target subset far from the query point. Finally, maintaining a separate KNN data-structure for each potential subset of the data is costly and can result in $2^n$ data-structures, where $n$ is the total number of images. In this work, we show that tree-based data-structures provide a natural way to improve the performance of CIR. More specifically, we prove that Random Projection Trees \citep{dasgupta2008random} can flexibly adapt to subsets of data through pruning. We use this insight to design a modification to existing tree-based KNN methods that allows them to quickly prune their structure to adapt to any subset of their original data using an inverted index. These structures outperform the commonly used CIR heuristics mentioned above. Finally, we investigate the structure of conditional KNN trees to show that they can reveal areas of poor convergence and diversity (``blind spots'') in image based GANs. We summarize the contributions of this work as follows:


\begin{itemize}
\itemsep-.2em 
    \item We introduce an interactive web application to discover connections across cultures, artists, and media in the visual arts.
    \item We prove an efficiency lower bound for solving CIR with pruned Random Projection trees.
    \item We contribute a strategy for extending existing KNN data-structures to allow users to efficiently filter resulting neighbors using arbitrary logical predicates, enabling efficient CIR.
    \item We show that CIR data-structures can discover ``blind spots'' where GANs fail to match the true data.
\end{itemize}

\begin{figure}
    \centering
  \includegraphics[width=.8\linewidth]{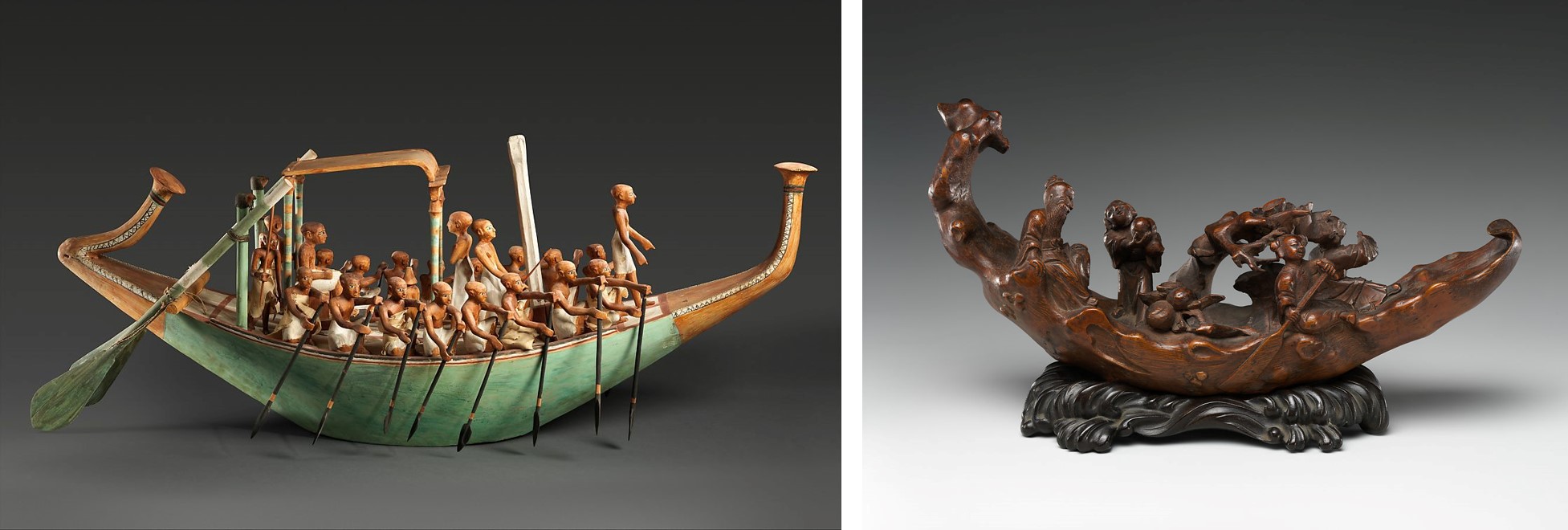}

  \captionof{figure}{A pair of cross cultural images found with CIR. Left: \textit{Model Paddling Boat} from 1980 BC Egypt. Right: \textit{Immortal Raft} from 18th Century China.}
  \label{fig:matched_boats}
\end{figure}

\section{Background} 

IR systems aim to retrieve a list of relevant images that are related to a query image. ``Relevance'' in IR systems often refers to the ``semantics'' of the image such as its content, objects, or meaning. Many existing IR systems map images to ``feature space'' where distance better corresponds to relevance. In feature space, KNN can provide a ranked list of relevant images \citep{manning2008introduction}. Good features and distance metrics aim to align with our intuitive senses of similarity between data \citep{yamins2014performance} and show invariance to certain forms of noise \citep{gordo2016deep}. 
There is a considerable body of work on learning good ``features" for images \citep{bengio2013representation, zhang2016colorful, radford2015unsupervised, koch2015siamese, huh2016makes}.
In this work we leverage features from intermediate layers of deep supervised models, which perform well in a variety of contexts and are ubiquitous throughout the literature. Nevertheless, our methods could apply to any features found in the literature including those from collaborative filtering, text, sound, and tabular data. 

\begin{figure*}[tbp]
    \centering
      \includegraphics[width=.9\linewidth]{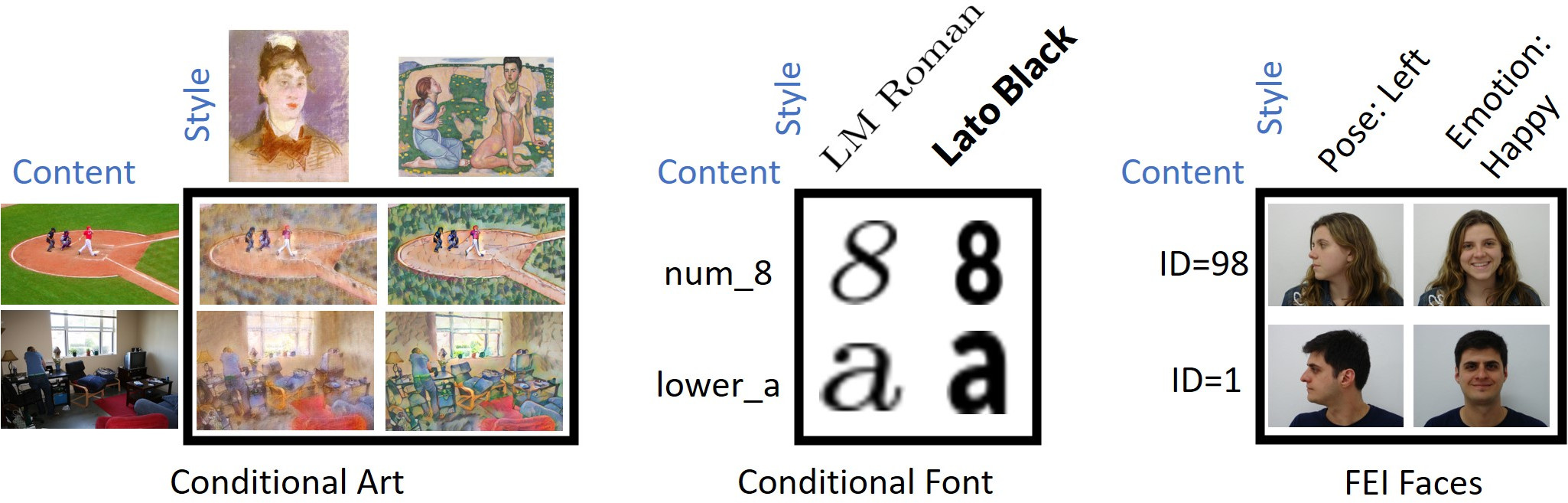}
  \caption{Representative samples from the ConditionalArt dataset (left), ConditionalFont dataset (middle), and FEI Face dataset (right). CIR systems conditioned on style should retrieve images of the same content.}
  \label{fig:content_style_ds}
\end{figure*}

There are a wide variety of KNN algorithms, each with their own strengths and weaknesses. Typically, these methods are either tree-based, graph-based, or hash-based \citep{knnbenchmark}. 
Tree-based methods partition target points into hierarchical subsets based on their spatial geometry and include techniques such as the KD Tree \citep{bentley1975multidimensional}, PCA Tree \citep{bachrach2014speeding}, Ball Tree \citep{omohundro1989five}, some inverted index approaches \citep{invertedindex}, and tree ensemble approaches \citep{yan2019k}. 
Some tree-based data-structures allow exact search with formal guarantees on their performance \citep{dasgupta2008random}. 
Graph-based methods rely on greedily traversing an approximate KNN graph of the data, and have gained popularity due to their superior performance in the approximate NN domain \citep{knnbenchmark, johnson2019billion}.
There are many hash-based approaches in the literature and \cite{hashing-summary} provides a systematic overview. 
To our knowledge, neither graph nor hash-based retrieval methods can guarantee finding the nearest neighbor deterministically. However, fast approximate search is often sufficient for many applications. In our work we focus on tree-based methods because it is unclear how to create an analogous method for graph-based data-structures. Nevertheless, tree-based methods are widely used, especially when exact results are needed. Surprisingly, conditional KNN systems have only received attention recently, even though conditional queries appear in shopping, search, and recommendation systems. To our knowledge, \citep{matsui2018reconfigurable} is the only effort to improve performance of these systems by adaptively switching from a ``query-then-filter" strategy to brute-force at a particular size threshold.

\section{Conditional Image Retrieval}

To generalize an IR system to handle queries over any image subset we generalize the KNN problem to this setting. More formally, the Conditional K-Nearest Neighbors (CKNNs) of a query point, $q$, are the $k$ closest points with respect to the distance function, $\delta$, that satisfy a given logical predicate (condition), $\mathcal{S}$. We represent this condition as a subset of the full corpus of points, $\mathcal{X}$: 

$$CNN(q, \mathcal{S} \subseteq \mathcal{X}) = \argmin_{t \in \mathcal{S}} \delta(q,t) $$

When the conditioner, $\mathcal{S}$, equals the full space, $\mathcal{X}$, we recover the standard KNN definition. Figure \ref{fig:example} shows a visualization of CKNN for a two-dimensional dataset with two classes. With conditional KNN queries it’s possible to combine logical predicates and filters with geometry-based ranking and retrieval. 

To create a CIR system, one can map images to a ``feature-space'', where distance is semantically meaningful, prior to finding CKNNs. One of the most common featurization strategies uses the penultimate activations of a supervised network such as ResNet50 \citep{he2016identity} trained on ImageNet \citep{deng2009imagenet}. 
Alternatively, using ``style'' based features from methods like AdaIN \citep{huang2017arbitrary} enable CIR systems that retrieve images by ``style'' as opposed to content. 
Deep features capture many aspects of image semantics such as texture, color, content, and pose \citep{olah2017feature} and KNNs in deep feature space are often both visually and semantically related. We aim to explore whether this observation holds for conditional matches across disparate subsets of images, which requires a more global feature-space consistency.

\begin{table*}[tbp]
\centering
\begin{tabular}{cl|ccccccccc}
\multicolumn{2}{c}{}          & \multicolumn{9}{|c}{Featurization Method}                                                  \\
Dataset              & Metric & RN50 & RN101 & MN  & SN           & DN           & RNext & dlv3101 & MRCNN        & Random \\ \hline\hline
\multirow{2}{*}{CA}  & @1     & .50  & .51   & .55 & .44          & \textbf{.59} & .46   & .37     & .45          & .0002  \\
                     & @10    & .70  & .68   & .71 & .62          & \textbf{.76} & .65   & .55     & .63          & .002   \\ \hline
\multirow{2}{*}{CF}  & @1     & .41  & .37   & .39 & \textbf{.44} & .43          & .38   & .33     & \textbf{.44} & .016   \\
                     & @10    & .77  & .76   & .76 & \textbf{.80} & .79          & .76   & .73     & .79          & .16    \\ \hline
\multirow{2}{*}{FEI} & @1     & .80  & .84   & .85 & .79          & \textbf{.87} & .86   & .72     & .78          & .005   \\
                     & @10    & .94  & .93   & .94 & .89          & \textbf{.95} & .94   & .86     & .92          & .05    \\ \hline
\end{tabular}
  \vspace{-.07in}
\caption{Performance of CIR (Accuracy $@N$) on content recovery across style variations for both the ConditionalFont (CF) and ConditionalArt (CA) datasets using a variety of features from pre-trained networks. Results show CIR retrieves the same content image across different styles.  For full details on experimental conditions see Section \ref{sec:experimental-details}}.
\label{table:cir_recall}
\end{table*}

\section{Discovering Shared Structure in Visual Art}

We find that CIR on the combined Met and Rijksmusem collections finds striking connections between art from different histories and mediums. These matches show that even across large gaps in culture and time CIR systems can find relevant visual and semantic relations between images. For example, Figure \ref{fig:matched_boats} demonstrates a pair of images that, despite being separated by 3 millennia and 7,000 Kilometers, have an uncanny visual similarity and cultural meaning. More specifically, both works play a role in celebrating and safeguarding passage into the afterlife \citep{werner1922myths, oppenheim2015ancient, hayes1990scepter}. Matches between cultures also highlight cultural exchange and shared inspiration. For example, the similar ornamentation of the Dutch Double Face Banyan (left) and the Chinese ceramic figurine (top row second from left) of Figure \ref{fig:feature-graphic} can be traced to the flow of porcelain and iconography from Chinese to Dutch markets during the 16\textsuperscript{th}-20\textsuperscript{th} centuries \citep{le1974china, volker1954porcelain}. CIR also provides a means for diversifying the results of visual search engines through highlighting conditional matches for cultures, media, or artists that are less frequently explored. We hope CIR can help the art-historical community and the public explore new artistic traditions. This is especially important during the COVID-19 pandemic as many cultural institutions cannot accept visitors. To this end, we introduce an interactive art CIR application, \url{aka.ms/mosaic}, and provide more details in Section \ref{sec:website}. In Section \ref{sec:additional-matches} of the Appendix we also provide additional examples and representative samples.

\section{The MosAIc Web Application}
\label{sec:website}

\begin{figure*}[h]
  \centering
  \includegraphics[width=\linewidth]{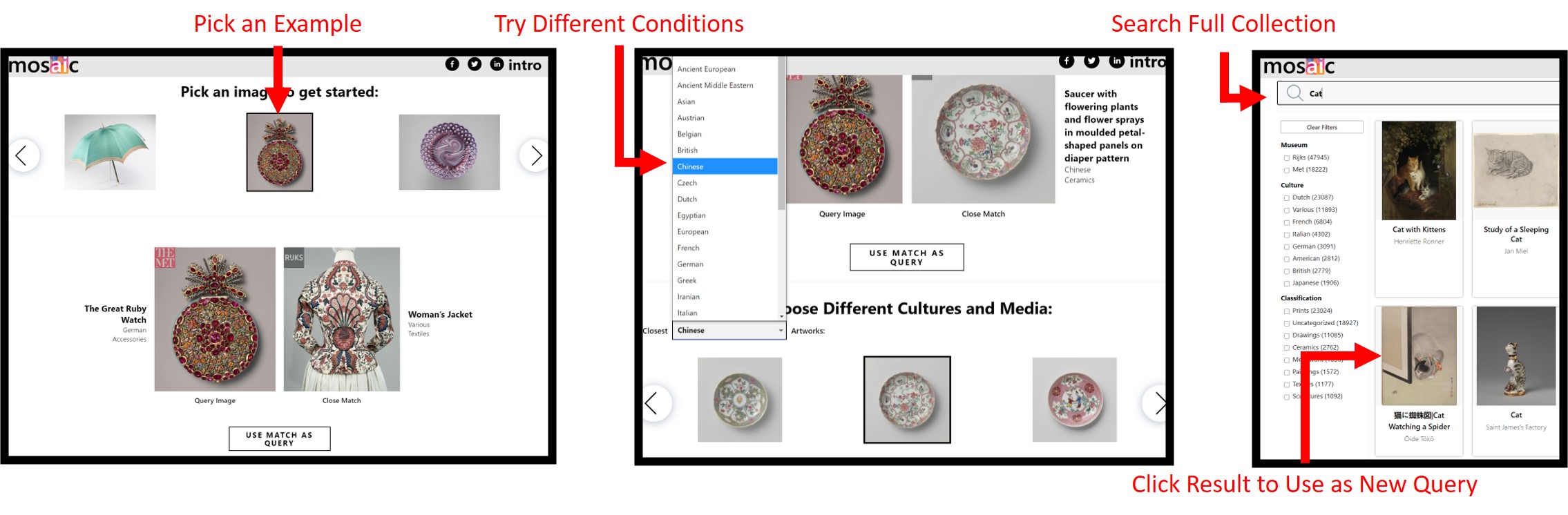}
  \caption{Using the MosAIc web application (\url{aka.ms/mosaic}). After watching a short video explaining the app, users can select a work of art to find conditional matches with (left). Users can find conditional matches for a variety of different cultures and media (middle). To further explore the collection, users can search for new query objects using a conventional search index (right). Users can also construct chains of conditional matches using the ``Use match as query'' button below the main matches.}
  \label{fig:mosaic1}
\end{figure*}

As an application of CIR for the public, we introduce MosAIc (\url{aka.ms/mosaic}), a website that allows users to explore art matches conditioned on culture and medium. Our website aims to show how conditional image retrieval can find surprising and uncanny pairs of artworks that span millennia. We also aim to make it easy for interested users to find new artworks in cultures they might not think to explore during a physical museum visit. Using the MosAIc application, users can choose from a wide array of example objects to use as conditional search queries as shown in the left panel of Figure \ref{fig:mosaic1}. Users can select from an array of different cultures and media to condition their searches as in Figure \ref{fig:mosaic1} middle. Selecting a specific medium or culture, allows the user to browse the top conditional matches in that category and use these matches as new query images. This enables traversing the collection using conditional searches to find relevant content in different areas of the collection. Additionally, for users who want to use a specific work of art as a starting point we have added a conventional text based search engine to quickly find specific works relating to a keyword as in Figure \ref{fig:mosaic1} right.

The mosaic application combines a React \citep{fedosejev2015react} front-end with a back-end built from Azure Kubernetes Service, Azure Search, and Azure App Services. Our front-end features responsive design principles to support for mobile, tablet, desktop, and ultra-wide displays. We also aim to use high-contrast design to make the application more accessible to the low-vision community. To create the conditional search index, we featurize the combines Metropolitan Museum of Art and Rijksmuseum open access collections using ResNet50 from torchvision \cite{marcel2010torchvision}. We then add these features to a Conditional Ball tree for real-time conditional retrieval and deploy this method as a RESTful service on Azure Kubernetes Service. Additionally, we store image metadata, automatically generated image captions, and detected objects in an Azure Search index which allows querying for additional information, and supports text search. To add captions and detected objects to over 500k images we use the Cognitive Services for Big Data \cite{hamilton2020large}.

\section{Evaluating CIR Quality}
\label{sec:eval}

Though finding connections between art is of great importance to the curatorial and historical communities, it is difficult to measure a system's success on this dataset as there are no ground truth on which images \textit{should} match. To understand the behavior of CIR systems quantitatively we investigate datasets with known content images aligned across several different ``styles'' or subsets to retrieve across. More specifically, if the conditioning information represents the image ``style'' and the features represent the ``content'', CIR should find an image with the same content, but constrained to the style of the conditioner, such as ``Ceramic'' or ``Egyptian'' in Figure \ref{fig:feature-graphic}. Through this lens, CIR systems can act as ``non-parametric'' style transfer systems. This approach differs from existing style transfer and visual analogy methods in the literature \citep{huang2017arbitrary, gatys2016image} as it does not generate new images, but rather it finds analogous images within an existing corpora.

To this end, we apply CIR to the FEI face database of 2800 high resolution faces across 200 participants and 14 poses, emotions, and lighting conditions. We also introduce two new datasets with known style and content annotations: the ConditionalFont and ConditionalArt datasets. The ConditionalFont dataset contains 15687 $32\times32$ grayscale images of 63 ASCII characters (content) across 249 fonts (style). The ConditionalArt dataset contains 1,000,000 color images of varying resolution formed by stylizing 5000 content images from the MS COCO \citep{lin2014microsoft} dataset with 200 style images from the WikiArt dataset \citep{nichol2016painter} using an Adaptive Instance Normalization \citep{huang2017arbitrary}. Although this dataset is ``synthetic'', \citep{Jing_2019} show that neural style transfer methods align with human intuition. We show representative samples from each dataset in Figure \ref{fig:content_style_ds}.

With these datasets it's possible to measure how CIR features, metrics, and query strategies affect CIR's ability to match content across styles.
To measure retrieval accuracy, we sampled 10000 random query images. For each random query image, we use CIR to retrieve the query image's KNNs conditioned on a random style. We then check whether any retrieved images have the same content as the original query image. In Table \ref{table:cir_recall}, we explore how the choice of featurization algorithm affects CIR systems. All methods outperform the random baseline of Table \ref{table:cir_recall}, indicating that they are implicitly performing non-parametric content-style transfer. DenseNet (DN) \citep{iandola2014densenet} and Squeezenet (SN) \citep{iandola2016squeezenet} tend to perform well across all datasets. CIR performs well across all three tasks \textit{without} fine tuning to the structure of the datasets, indicating that this approach can apply to other zero-shot image-to-image matching problems.

\section{Fast CKNN with Adaptive Tree Pruning}
\label{sec:improving}

\begin{figure}[t]
    \centering
    \includegraphics[width=\linewidth]{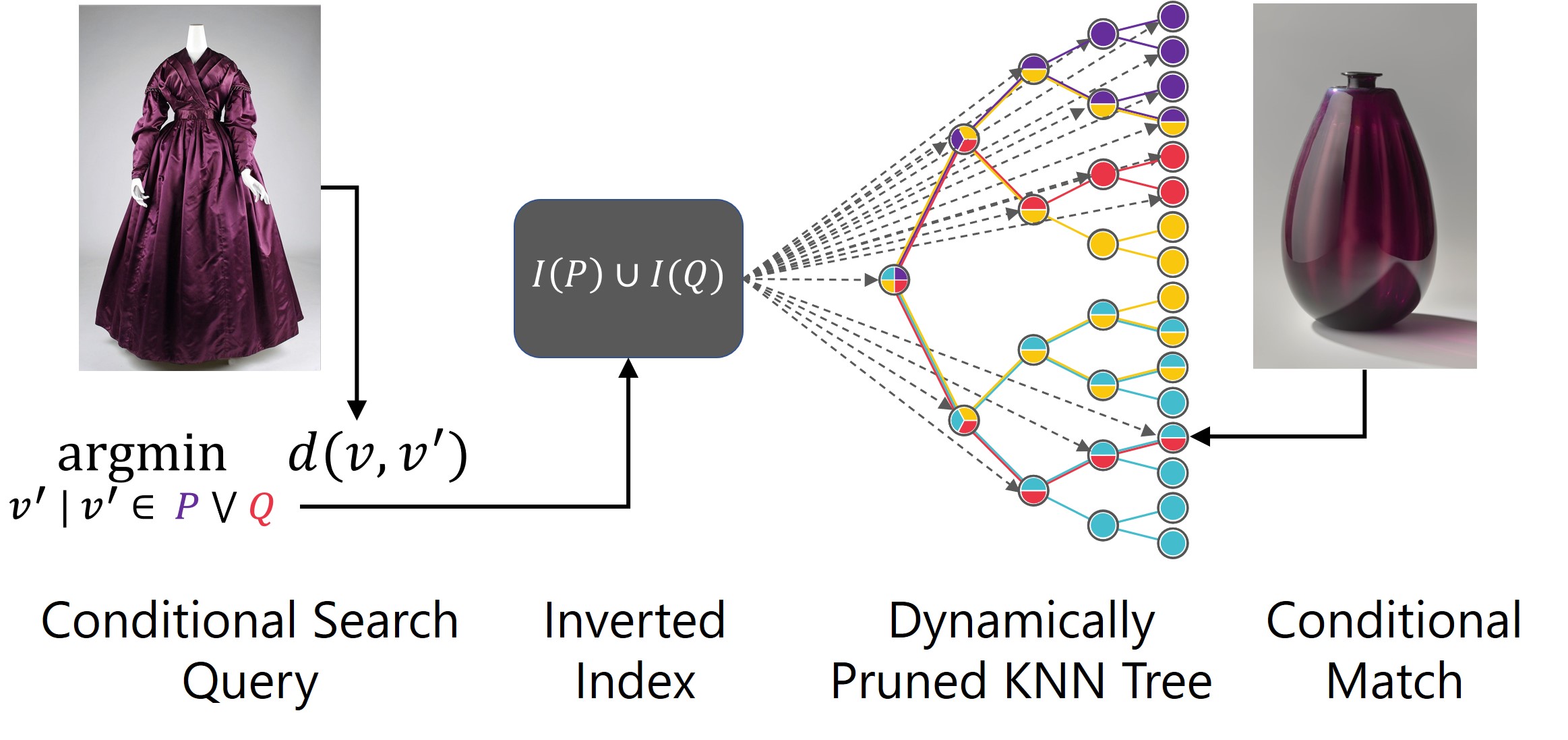}
    \caption{Dynamic tree pruning based CIR architecture. The user specified condition, $P \lor Q$, is translated to an inverted index query and the result is used to prune the unconditional KNN tree where nodes are colored based on which conditions they contain. This pruned tree accelerates conditional search for any subset by reducing the number of nodes considered in tree traversal.}
    \label{fig:pruning}
\end{figure}

In Section \ref{sec:eval} we have shown that CIR is semantically meaningful in several different contexts, but the question remains as to whether this approach affords an efficient implementation that can scale to large datasets with low latency. Conventional IR systems scale to this setting using dedicated data-structures such as trees, spatial hashes, or graphs. There are a wide variety of strategies with provable guarantees in the unconditional setting, but it is not known if existing data-structures can apply naturally to the conditional setting. In this work we focus on extending tree-based methods to the conditional setting.  Tree-based methods are some of the only methods that guarantee \textit{exact} KNN retrieval, and there are already several theoretical results on the performance of these methods \citep{dasgupta2008random, dhesi2010random}. In particular, \citep{dasgupta2008random} show that RandomProjection-Max (RP) trees can adapt to the intrinsic dimensionality of the data and prove bounds that demonstrate the effectiveness of the data-structure. \citep{dhesi2010random} continue this line of reasoning and prove a packing lemma using a bound on the aspect ratio of RP tree cells. These works show that RP trees are effective at capturing the geometry of the training data. Our aim is to show that they also capture the geometry of \textit{subsets} of the training data through their sub-trees. More specifically, we show that for any subset of the training data, one can derive probabilistic bound the number of nodes in the tree that contain this subset. More formally:

\begin{theorem} Suppose an \textsc{RPTree-Max}, $\mathcal{T}$, is built using a dataset $\mathcal{X} \subset \mathbb{R}^D$, of diameter $W$, with doubling dimension $\leq d$. Further suppose $\mathcal{T}$ is balanced with a cell-size reduction rate bounded above by $\gamma$. Let $\mathcal{S} \subseteq \mathcal{X}$ be a subset of the dataset used to build the tree and $\mathcal{B}$ a finite set of radius $R>0$ balls that cover $\mathcal{S}$. For every $0< \epsilon < 1$ there exists a constant, $c > 0$, such that with probability $> 1 - \epsilon$ the fraction of cells that contain points within $\mathcal{S}$ is bounded above by $|\mathcal{B}|2^{-log_\gamma(W/R')}$ where $R' = c R d \sqrt{d} \log(d)$
\end{theorem}

We point readers to \citep{dhesi2010random}, for the precise definition of an RPTree, cell-size, and the doubling dimension. To sketch the proof, we first generalize an aspect bound from \citep{dhesi2010random} to show that, with high probability, small radius balls can be completely inscribed within small radius RP tree cells. Because it takes several levels before the tree's cells shrink to this size, we can bound this cell's depth and thus the size of its sub-tree relative to the full tree. By considering a collection of balls that cover our target subset, we arrive at the final bound. See section \ref{sec:proof} of the Appendix for a full proof. 

\begin{algorithm2e}[t]
\caption{Querying a CKNN Tree}\label{algo:query}
\SetKwData{Left}{left}\SetKwData{This}{this}\SetKwData{Up}{up}
\SetKwFunction{Union}{Union}\SetKwFunction{FindCompress}{FindCompress}
\SetKwInOut{Input}{input}\SetKwInOut{Output}{output}
\SetKwFunction{FSearchNode}{SearchNode}
\SetKwProg{Fn}{def}{:}{}
\Input{A point, $q$, a condition,\ $\mathcal{S} \subseteq \mathcal{X}$, a tree, $root$, and an inverted index, $I$}
\Output{Closest point, $p^* \in \mathcal{S}$, to $q$}
$validNodes \leftarrow \bigcup_{s \in \mathcal{S}} I(s)$;
$p^* \leftarrow$ null

\Fn{\FSearchNode{$n$}}{
 
    \If{$n \in validNodes$}{
 
        \eIf{$n$ is a leaf node}{

            $p \leftarrow$ closest point in $\mathcal{S}$
    
            \If{$d(p, q) < d(p^*, q)$}{$p^* \leftarrow p$}

        }{

          $potentials \leftarrow$ children of $n$ which could hold a closer point

          \For{$child$ in $potentials$}{\FSearchNode{child}}
        }
    }

}
\FSearchNode{root}; \KwRet $p*$
\end{algorithm2e}

\begin{figure}[t]
\centering
  \includegraphics[width=.8\linewidth]{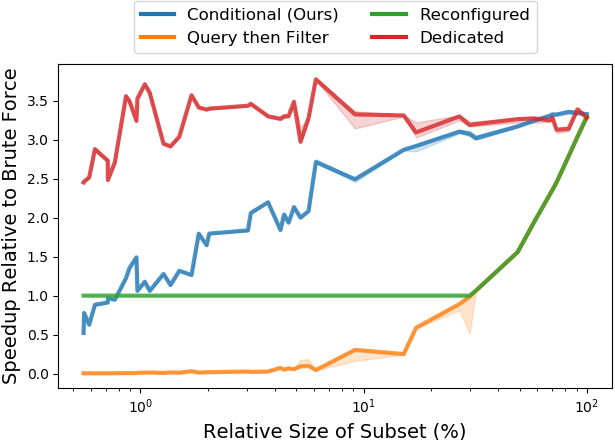}
    \vspace{-.07in}
  \caption{Query time of Conditional KNN approaches. Our approach (Conditional) achieves query performance approaching that of an tree recreated specifically for each query (Dedicated) \textit{without the expensive re-creation cost}, and does not perform poorly with small conditions like ``Query then Filter'' strategies. Furthermore, our method accelerates queries across much smaller subsets than the reconfiguration strategy of \citep{matsui2018reconfigurable}. Please see Section \ref{sec:experimental-details} and \ref{sec:improving} for method details.}
  \label{fig:cknn-speedup}
\end{figure}

This theorem not only shows that sub-trees of an RP tree capture the geometry of training dataset subsets, but also points to a method to improve the speed of CKNN. Namely, we can prune tree nodes that do not hold points within our target subset prior searching for conditional neighbors. We diagram this procedure in Figure \ref{fig:pruning}, and provide pseudo-code in Algorithm \ref{algo:query}. We now turn our attention to quickly computing the proper sub-trees for each subset of the data. To this end, one can use an inverted index \citep{knuth1997art}, $I$, that maps points, $x \in \mathcal{X}$ to the collection of their dominating nodes, $I(x) = \{n : x \text{ below node } n \}$. One can compute the subset of nodes that remain after pruning by taking the union of dominating nodes as shown in the first line of Algorithm \ref{algo:query} and in the illustration of the full search architecture in Figure \ref{fig:pruning}. Evaluating the predicate on points within leaf nodes can also reduce computation. 

Additionally, if the predicates of interest have additional structure, such as representing class labels, one can define a smaller class-based inverted index, $I_{class}(c)$ which maps a class label, $c$, to the set of dominating nodes. For these predicates, union and intersection operators commute through the class-based inverted index: 

\begin{equation} \label{eqn:commute}
\begin{split}
    I(\mathcal{S}_a \cap \mathcal{S}_b) &= I_{class}(a) \cap I_{class}(b)\\
    I(\mathcal{S}_a \cup \mathcal{S}_b) &= I_{class}(a) \cup I_{class}(b)
\end{split}
\end{equation}

where $\mathcal{S}_a$ is the subset of points with label $a$. This principle speeds a broad class of queries and accelerates document retrieval frameworks like ElasticSearch \citep{gormley2015elasticsearch} and its backbone, Lucene \citep{mccandless2010lucene}. We stress that this approach does not use Lucene to filter images or documents directly, but rather to filter nodes of a KNN retrieval data-structure at query time. This enables a rich ``predicate push-down`` \citep{hellerstein1993predicate, levy1994query} logic for KNN methods independent of how the tree splits points (Ball, Hyperplane, Cluster), the branching factor, and the topology of the tree. It also applies to ensembles of trees and to multi-probe LSH methods by pruning hash buckets. We note that our proposed indexing structure is small compared to the size of the underlying dataset, and unconditional KNN tree, and provide an analysis of memory footprints in Table \ref{table:mem}. 

\section{Performance}

In Figure \ref{fig:cknn-speedup}, we show the relative performance of several strategies for CIR on 488k Resnet50-featurized images ($dim=2048$) from the combined MET and Rijksmusem open-access collections with a randomly chosen test set ($n=1000$). We condition on artwork media, culture, and several combinations of these to create a variety of condition sizes. We measure the speedup compared to a vectorized Brute-Force search using NumPy arrays \citep{walt2011numpy}. We implement CKNN methods with respect to one of the most used implementations of KNN, Sci-kit Learn's Ball Tree algorithm \citep{pedregosa2011scikit}. We compare our approach (Conditional) to, the standard ``query-then-filter'' approach, and adaptive switching to brute force search (Reconfigured) \citep{matsui2018reconfigurable}. Finally, we compare to a ``best-case'' scenario of a KNN data-structure pre-computed for every subset (Dedicated). Though in practice it is often impossible to make an index for each subset, this setting provides an upper bound on the performance of any approach. Our analysis shows that adaptive pruning (Conditional) outperforms other approaches and is close to optimal for large subsets of the dataset. Additionally, the performance of the ``Query-then-filter'' strategy quickly degrades for small subsets of the dataset as expected. Our approach is also compatible with prior work on adaptively switching to brute force and allows one to set the ``switch-point'' over 10x lower. We also note that these results hold with randomized conditions, and across other similar datasets.

Finally, we stress that the goal of this work is \textbf{not} to make the fastest unsupervised KNN method, but rather to evaluate generic strategies to transform these approaches to the conditional setting. There is a considerable body of work on fast, \textit{approximate, unconditional} KNN methods which often outperform Scikit Learn's exact retrieval algorithms. We point readers to \citep{knnbenchmark} for more details. We stress that exhaustive benchmarking of unconditional KNN indices and approaches is outside the scope of this work. For implementation, experimentation, environment, and computing details please see Section \ref{sec:experimental-details}.

\begin{figure*}[tbp]
      \includegraphics[width=\linewidth]{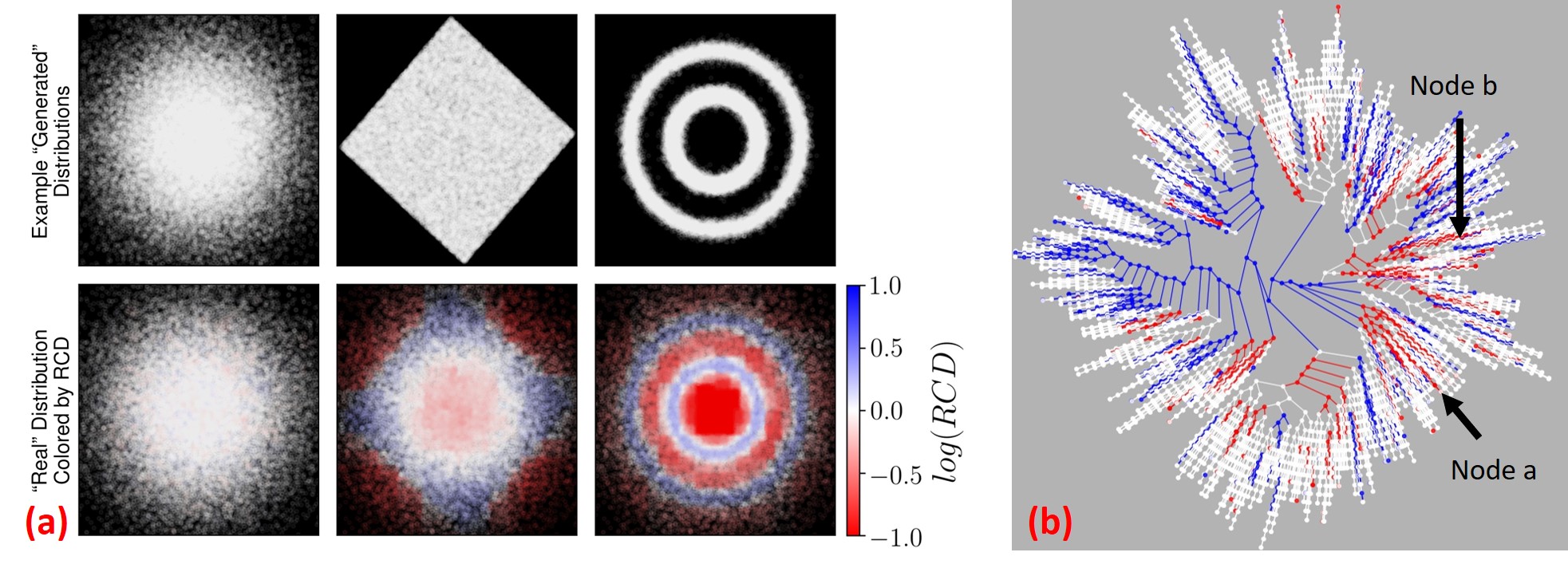}
        \vspace{-.07in}
      \caption{
      (a): Visualization of the RCD between several example distributions and a standard normal of ``real'' data ($n=50k$). Upper plots show generated distributions, and bottom plots show the ``real'' distribution colored by the RCD induced by a CKNN Tree. Even though these datasets are identical under the popular Frech\'et inception distance (FID), the RCD detects areas where generated data over (blue) and under (red) samples the real data.
      (b): Nodes of a CKNN tree (Center node is the root) colored by statistically significant deviations of RCD from 1 ($p < 0.01$). This shows widespread differences between GAN outputs and true data. Red nodes represent areas where the GAN under samples the empirical distribution, and blue nodes over-sample. High discrepancy nodes $a$ and $b$ from Figure \ref{fig:bad_gan_samples} are annotated.}
      \label{fig:gan_tree}
\end{figure*}

\subsection{Implementation}

We implement adaptive tree pruning for the existing Ball Tree and KD tree implementations in the popular SciKit-learn framework. Our implementation supports exact retrieval with several metrics, OpenMP parallelization \citep{dagum1998openmp}, and Cython acceleration\citep{behnel2011cython}. We also provide accelerations such as dense bit-array set operations, and caching node subsets on repeated conditioner queries. For larger scale datasets, we contribute a Spark based implementation of a Conditional Ball Tree to Microsoft ML for Apache Spark \citep{hamilton2018flexible, hamilton2018mmlspark}. 

To enable integration with differentiable architectures common in the community, we provide a high-throughput, PyTorch module \citep{paszke2017automatic} for CIR. This implementation is fundamentally brute force but uses Einstein-summation to retrieve conditional neighbors for multiple queries and multiple conditions simultaneously, and can increase throughput by over 100x compared to naive PyTorch implementations.

\section{Limitations}
This work does not aim to create the fastest KNN algorithm, but rather presents a formally motivated technique to speed up existing tree-based KNN methods in the conditional setting. 
KNN retrieval chooses some items significantly more than others, due to effects such as the ``hubness problem'' and we direct readers to \citep{dinu2014improving} for possible solutions. 
We present additional diversity reducing geometries in Section \ref{sec:low-dim-failure} of the Appendix.
Our approach does not modify the KNN construction, simply prunes it afterwards. This may not be the most efficient solution when conditioner sizes are small, but it is orders of magnitude faster than recreating the tree. We also note that the performance of our conditional KNN methods are dependent on the underlying unconditional KNN tree, which often performs better on datasets with smaller intrinsic dimension. 

\section{Discovering ``Blind Spots'' in GANs}
\label{sec:blind-spot}

Efficient high-dimensional KNN search data-structures adapt to the geometry and intrinsic dimensionality of the dataset \citep{dasgupta2008random,dhesi2010random}. Moreover, some recent KNN methods use approaches from unsupervised learning like hierarchical clustering \citep{wang2011fast} and slicing along PCA directions \citep{bachrach2014speeding}. In this light, CKNN trees allow us to measure and visualize the ``heterogeneity'' of conditioning information within a larger dataset. More specifically, by analyzing the relative frequency of labels within the nodes of a CKNN tree, one can find areas with abnormally high and low label density. More formally, we introduce the Relative Conditioner Density (RCD) to measure the degree of over or under representation of a class $c$ with corresponding subset $\mathcal{S}_c \subseteq \mathcal{X}$, at node $n$ in the KNN tree:

\begin{equation}
    RCD(n, c) = \frac{|n \cap \mathcal{S}_c|}{|n|}\frac{|\mathcal{X}|}{|\mathcal{S}_c|}
    \label{eqn:lpf}
\end{equation}

Here, $|n|$ is the number of points below node $n$ in the tree. The RCD measures how much a node's empirical distribution of labels differs from that of the full dataset. $RCD >1$ occurs when the node over-represents class $c$, and $RCD <1$ occurs when the node under-represents a class, $c$. We apply this statistic to understand how samples from generative models, such as image-based GANs, differ from true data. In particular, one can form a conditional tree containing true data and generated samples, each with their own classes, $c_t$ and $c_g$ respectively. In this context, nodes with $RCD( \cdot, c_{g}) \ll 1$ are regions of space where the network under-represents the real dataset. To illustrate this effect, Figure \ref{fig:gan_tree}a shows several simple 2d examples. Even though these datasets are identical with respect to the Fr\'echet Distance \citep{heusel2017gans}, coloring points based on their parent node RCD's can highlight areas of over and under sampling of the true distribution by each ``generated'' distribution. In Figure \ref{fig:gan_tree}b, we form a CKNN tree on samples from a trained Progressive GAN \citep{karras2017progressive} and it's training dataset, CelebA HQ \citep{liu2015faceattributes}. Coloring the nodes by RCD reveals a considerable amount of statistically significant structural differences between the two distributions. By simply thresholding the RCD ($< 0.6$), we find types of images that GANs struggle to reproduce. We show samples from two low-RCD nodes in Figure \ref{fig:bad_gan_samples} and also note their location in Figure \ref{fig:gan_tree}b. Within these nodes, Progressive GAN struggles to generate realistic images of brimmed hats and microphones. Though we do not focus this work on thoroughly investigating issues of diversity in GANs, this suggests GANs have difficulty representing data that is not in the majority. This aligns with the findings of \citep{bau2019seeing}, without requiring GAN inversion, additional object detection labels, or a semantic segmentation ontology. Furthermore, we note that the FID cannot capture the full richness of why two distributions differ, as this metric just measures differences between high dimensional means and co-variances. Using CKNN trees can offer more flexible and interpretable ways to understand the differences between two high dimensional distributions.


\begin{figure*}[tbp]
\includegraphics[width=\linewidth]{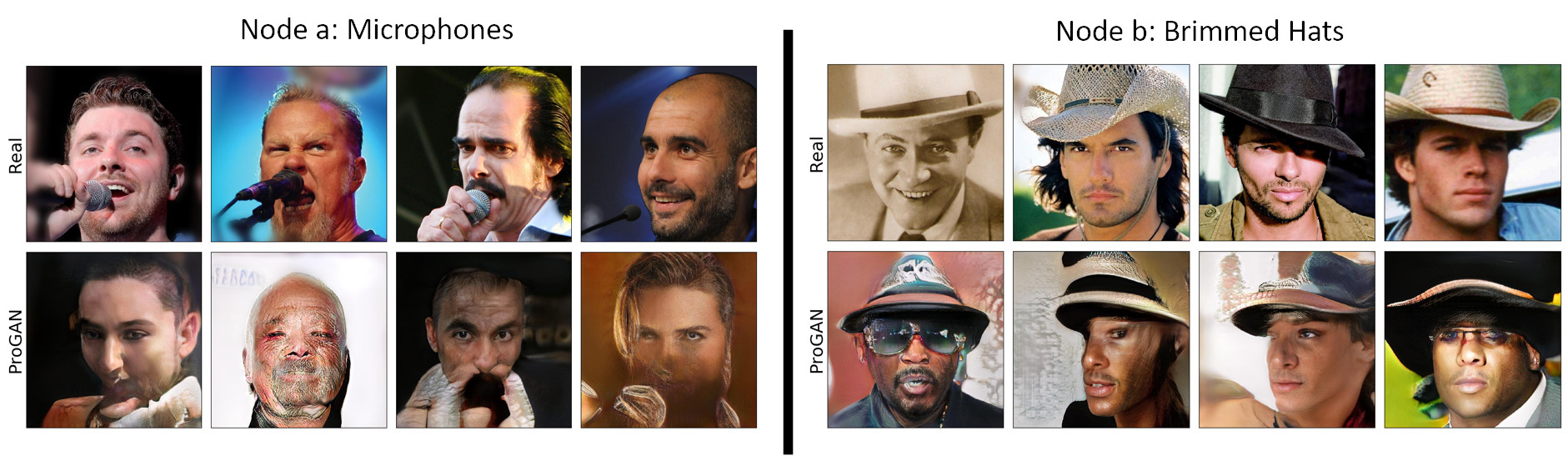}
  \vspace{-.2in}
  \caption{Samples from two statistically significant nodes from Figure \ref{fig:gan_tree}. Images are randomly chosen and representative of those found at the node. Almost every real image in Node a contains microphones whereas no GAN generated outputs could create a microphone. Node b shows a clear bias towards brimmed hats, and the GAN samples have significant visual artifacts.}
  \label{fig:bad_gan_samples}
\end{figure*}

\section{Experimental Details}
\label{sec:experimental-details}

All experiments use an Ubuntu 16.04 Azure NV24 Virtual Machine with Python 3.7 and scikit-learn v0.22.2 \citep{pedregosa2011scikit}. We use scikit-learn's Ball Tree and KD Tree and use numpy v1.18.1 \citep{walt2011numpy} for brute force retrieval.
For query-then-filter strategies we first retrieve 50 points, then increase geometrically (x5) if the query yeilds no valid matches.
To form image features for Table \ref{table:cir_recall}, we use trained networks from torchvision v0.6 \citep{marcel2010torchvision}. In particular, we use ResNet50 (RN50) \citep{he2016identity}, ResNet101 (RN101), MobileNetV2 (MN) \citep{mnv2}, SqueezeNet (SN) \citep{iandola2016squeezenet}, DenseNet (DN) \citep{iandola2014densenet}, ResNeXt (RNext) \citep{resnext}, DeepLabV3 ResNet101 (dlv3101) \citep{dlv3}, and Mask R-CNN (MRCNN) \citep{mrcnn}. Features are taken from the penultimate layer of the backbone, and the matches of Table \ref{table:cir_recall} are computed with respect to cosine distance. We use trained a Progressive GAN from the open-source Tensorflow 
implementation accompanying \citep{karras2017progressive}.

\section{Related Work}

Image retrieval and nearest neighbor methods have been thoroughly studied in the literature, but we note that the conditional setting has only received attention recently. There are several survey works on KNN retrieval, but they only mention unconditional varieties \citep{bhatia2010survey, hashing-summary}. \citep{marchiori2009class} has studied the mathematical properties of conditional nearest neighbor classifiers but works primarily with graph based methods as opposed to trees. They do not apply this to modern deep features and do not aim to improve query speed. There are a wide variety featurization strategies for IR systems. Gordo et. al \citep{gordo2016deep} learn features optimized for IR. Siamese networks such as FaceNet embed data using tuples of two data and a similarity score and preserving this similarity in the embedding \citep{koch2015siamese,Schroff_2015_CVPR}. Features from these methods could improve CIR systems. Conditional Similarity Networks augment tuple embedding approaches with the ability to handle different notions of similarity with different embedding dimensions \citep{veit2016conditional}. This models conditions as similarities but does not generically restrict the search space of retrieved images to match a user's query. These features have potential to yield neighbor trees that, when pruned, have a similar structure and performance to dedicated trees. Sketch-based IR uses line-drawings as query-images but does not aim to restrict the set of candidate images generically \citep{sketchretrieval}. Style transfer \citep{style-review} and visual analogies \citep{liao2017visual} yield results like our art exploration tool but generate the analogous images rather than retrieve them from an existing corpus. \citep{conditional-retrieval} split IR systems into conditional subsystems, but do not tackle generic conditioners or provide experimental evaluation. \citep{plummer2018conditional} create an IR system conditioned on text input, but do not address the problem of generically filtering results. \citep{gao2020interpretable} and \citep{liao2018interpretable} respectively learn and use a hierarchy of concepts concurrently with IR features, which could be a compelling way to \textit{learn} useful conditions for a Conditional IR system.

%

\section{Conclusion}

We have shown that Conditional Image Retrieval yields new ways to find visually and semantically similar images across corpora. We presented a novel approach for discovering hidden connections in large corpora of art and have creates an interactive web application, MosAIc to allow the public to explore the technique. We have shown that CIR performs non-parametric style transfer on the FEI faces and two newly introduced datasets. We proved a bound on the number of nodes that can be pruned from RandomProjection trees when focusing on subsets of the training data and used this insight to develop a general strategy for generalizing tree-based KNN methods to the conditional setting. We demonstrated that this approach speeds conditional queries and outperforms baselines. Lastly, we showed that CKNN data-structures can find and quantify subtle discrepancies between high dimensional distributions and used this approach to identify several ``blind spots'' in the ProGAN network trained on CelebA HQ. 

\acks{We would like to thank the Microsoft Garage program for supporting the development of the MosAIc application especially Chris Templeman, Linda Thackery, and Jean-Yves Ntamwemezi. Additionally we would like to thank Anand Raman, Markus Weimar, and Sudarshan Raghunathan for their feedback on the work and for their support of the work. }

\bibliography{jmlr-sample}

\begin{thebibliography}{77}
\providecommand{\natexlab}[1]{#1}
\providecommand{\url}[1]{\texttt{#1}}
\expandafter\ifx\csname urlstyle\endcsname\relax
  \providecommand{\doi}[1]{doi: #1}\else
  \providecommand{\doi}{doi: \begingroup \urlstyle{rm}\Url}\fi

\bibitem[Met(2019)]{Met}
{The Metropolitan Museum of Art Open Access CSV}, 2019.
\newblock URL \url{https://github.com/metmuseum/openaccess}.

\bibitem[Rij(2019)]{Rijks}
{The Rijksmuseum Open Access API}, 2019.
\newblock URL \url{https://data.rijksmuseum.nl/}.

\bibitem[Aum{\"{u}}ller et~al.(2018)Aum{\"{u}}ller, Bernhardsson, and
  Faithfull]{knnbenchmark}
Martin Aum{\"{u}}ller, Erik Bernhardsson, and Alexander~John Faithfull.
\newblock Ann-benchmarks: {A} benchmarking tool for approximate nearest
  neighbor algorithms.
\newblock \emph{CoRR}, abs/1807.05614, 2018.
\newblock URL \url{http://arxiv.org/abs/1807.05614}.

\bibitem[Bachrach et~al.(2014)Bachrach, Finkelstein, Gilad-Bachrach, Katzir,
  Koenigstein, Nice, and Paquet]{bachrach2014speeding}
Yoram Bachrach, Yehuda Finkelstein, Ran Gilad-Bachrach, Liran Katzir, Noam
  Koenigstein, Nir Nice, and Ulrich Paquet.
\newblock Speeding up the xbox recommender system using a euclidean
  transformation for inner-product spaces.
\newblock In \emph{Proceedings of the 8th ACM Conference on Recommender
  systems}, pages 257--264, 2014.

\bibitem[Baranchuk et~al.(2018)Baranchuk, Babenko, and Malkov]{invertedindex}
Dmitry Baranchuk, Artem Babenko, and Yury Malkov.
\newblock Revisiting the inverted indices for billion-scale approximate nearest
  neighbors.
\newblock \emph{CoRR}, abs/1802.02422, 2018.
\newblock URL \url{http://arxiv.org/abs/1802.02422}.

\bibitem[Bau et~al.(2019)Bau, Zhu, Wulff, Peebles, Strobelt, Zhou, and
  Torralba]{bau2019seeing}
David Bau, Jun-Yan Zhu, Jonas Wulff, William Peebles, Hendrik Strobelt, Bolei
  Zhou, and Antonio Torralba.
\newblock Seeing what a gan cannot generate, 2019.

\bibitem[Behnel et~al.(2011)Behnel, Bradshaw, Citro, Dalcin, Seljebotn, and
  Smith]{behnel2011cython}
Stefan Behnel, Robert Bradshaw, Craig Citro, Lisandro Dalcin, Dag~Sverre
  Seljebotn, and Kurt Smith.
\newblock Cython: The best of both worlds.
\newblock \emph{Computing in Science \& Engineering}, 13\penalty0 (2):\penalty0
  31--39, 2011.

\bibitem[Bengio et~al.(2013)Bengio, Courville, and
  Vincent]{bengio2013representation}
Yoshua Bengio, Aaron Courville, and Pascal Vincent.
\newblock Representation learning: A review and new perspectives.
\newblock \emph{IEEE transactions on pattern analysis and machine
  intelligence}, 35\penalty0 (8):\penalty0 1798--1828, 2013.

\bibitem[Bentley(1975)]{bentley1975multidimensional}
Jon~Louis Bentley.
\newblock Multidimensional binary search trees used for associative searching.
\newblock \emph{Communications of the ACM}, 18\penalty0 (9):\penalty0 509--517,
  1975.

\bibitem[Bhatia et~al.(2010)]{bhatia2010survey}
Nitin Bhatia et~al.
\newblock Survey of nearest neighbor techniques.
\newblock \emph{arXiv preprint arXiv:1007.0085}, 2010.

\bibitem[Bing(2017)]{bing_2017}
Bing.
\newblock Beyond text queries: Searching with bing visual search, Jun 2017.
\newblock URL
  \url{https://blogs.bing.com/search-quality-insights/2017-06/beyond-text-queries-searching-with-bing-visual-search}.

\bibitem[Chen et~al.(2017)Chen, Papandreou, Schroff, and Adam]{dlv3}
Liang{-}Chieh Chen, George Papandreou, Florian Schroff, and Hartwig Adam.
\newblock Rethinking atrous convolution for semantic image segmentation.
\newblock \emph{CoRR}, abs/1706.05587, 2017.
\newblock URL \url{http://arxiv.org/abs/1706.05587}.

\bibitem[Dagum and Menon(1998)]{dagum1998openmp}
Leonardo Dagum and Ramesh Menon.
\newblock Openmp: an industry standard api for shared-memory programming.
\newblock \emph{IEEE computational science and engineering}, 5\penalty0
  (1):\penalty0 46--55, 1998.

\bibitem[Dasgupta and Freund(2008)]{dasgupta2008random}
Sanjoy Dasgupta and Yoav Freund.
\newblock Random projection trees and low dimensional manifolds.
\newblock In \emph{Proceedings of the fortieth annual ACM symposium on Theory
  of computing}, pages 537--546, 2008.

\bibitem[DeGenova(2017)]{degenova2017}
Vinny DeGenova.
\newblock Recommending visually similar products using content based features,
  Dec 2017.
\newblock URL
  \url{https://tech.wayfair.com/data-science/2017/12/recommending-visually-similar-products-using-content-based-features/}.

\bibitem[Deng et~al.(2009)Deng, Dong, Socher, Li, Li, and
  Fei-Fei]{deng2009imagenet}
Jia Deng, Wei Dong, Richard Socher, Li-Jia Li, Kai Li, and Li~Fei-Fei.
\newblock Imagenet: A large-scale hierarchical image database.
\newblock In \emph{2009 IEEE conference on computer vision and pattern
  recognition}, pages 248--255. Ieee, 2009.

\bibitem[Dhesi and Kar(2010)]{dhesi2010random}
Aman Dhesi and Purushottam Kar.
\newblock Random projection trees revisited.
\newblock In \emph{Advances in Neural Information Processing Systems}, pages
  496--504, 2010.

\bibitem[Dinu et~al.(2014)Dinu, Lazaridou, and Baroni]{dinu2014improving}
Georgiana Dinu, Angeliki Lazaridou, and Marco Baroni.
\newblock Improving zero-shot learning by mitigating the hubness problem.
\newblock \emph{arXiv preprint arXiv:1412.6568}, 2014.

\bibitem[Fedosejev(2015)]{fedosejev2015react}
Artemij Fedosejev.
\newblock \emph{React. js essentials}.
\newblock Packt Publishing Ltd, 2015.

\bibitem[Gao et~al.(2020)Gao, Mu, Goulermas, Thiyagalingam, and
  Wang]{gao2020interpretable}
Xinjian Gao, Tingting Mu, John~Yannis Goulermas, Jeyarajan Thiyagalingam, and
  Meng Wang.
\newblock An interpretable deep architecture for similarity learning built upon
  hierarchical concepts.
\newblock \emph{IEEE Transactions on Image Processing}, 29:\penalty0
  3911--3926, 2020.

\bibitem[Gatys et~al.(2016)Gatys, Ecker, and Bethge]{gatys2016image}
Leon~A Gatys, Alexander~S Ecker, and Matthias Bethge.
\newblock Image style transfer using convolutional neural networks.
\newblock In \emph{Proceedings of the IEEE conference on computer vision and
  pattern recognition}, pages 2414--2423, 2016.

\bibitem[Gordo et~al.(2016)Gordo, Almaz{\'a}n, Revaud, and
  Larlus]{gordo2016deep}
Albert Gordo, Jon Almaz{\'a}n, Jerome Revaud, and Diane Larlus.
\newblock Deep image retrieval: Learning global representations for image
  search.
\newblock In \emph{European conference on computer vision}, pages 241--257.
  Springer, 2016.

\bibitem[Gormley and Tong(2015)]{gormley2015elasticsearch}
Clinton Gormley and Zachary Tong.
\newblock \emph{Elasticsearch: the definitive guide: a distributed real-time
  search and analytics engine}.
\newblock " O'Reilly Media, Inc.", 2015.

\bibitem[Grave et~al.(2018)Grave, Joulin, and Berthet]{grave2018unsupervised}
Edouard Grave, Armand Joulin, and Quentin Berthet.
\newblock Unsupervised alignment of embeddings with wasserstein procrustes,
  2018.

\bibitem[Hamilton et~al.(2018{\natexlab{a}})Hamilton, Raghunathan, Annavajhala,
  Kirsanov, Leon, Barzilay, Matiach, Davison, Busch, Oprescu,
  et~al.]{hamilton2018flexible}
Mark Hamilton, Sudarshan Raghunathan, Akshaya Annavajhala, Danil Kirsanov,
  Eduardo Leon, Eli Barzilay, Ilya Matiach, Joe Davison, Maureen Busch, Miruna
  Oprescu, et~al.
\newblock Flexible and scalable deep learning with mmlspark.
\newblock In \emph{International Conference on Predictive Applications and
  APIs}, pages 11--22, 2018{\natexlab{a}}.

\bibitem[Hamilton et~al.(2018{\natexlab{b}})Hamilton, Raghunathan, Matiach,
  Schonhoffer, Raman, Barzilay, Rajendran, Banda, Hong, Knoertzer,
  et~al.]{hamilton2018mmlspark}
Mark Hamilton, Sudarshan Raghunathan, Ilya Matiach, Andrew Schonhoffer, Anand
  Raman, Eli Barzilay, Karthik Rajendran, Dalitso Banda, Casey~Jisoo Hong,
  Manon Knoertzer, et~al.
\newblock Mmlspark: Unifying machine learning ecosystems at massive scales.
\newblock \emph{arXiv preprint arXiv:1810.08744}, 2018{\natexlab{b}}.

\bibitem[Hamilton et~al.(2020)Hamilton, Gonsalves, Lee, Raman, Walsh, Prasad,
  Banda, Zhang, Zhang, and Freeman]{hamilton2020large}
Mark Hamilton, Nick Gonsalves, Christina Lee, Anand Raman, Brendan Walsh,
  Siddhartha Prasad, Dalitso Banda, Lucy Zhang, Lei Zhang, and William~T
  Freeman.
\newblock Large-scale intelligent microservices.
\newblock \emph{arXiv preprint arXiv:2009.08044}, 2020.

\bibitem[Hayes(1990)]{hayes1990scepter}
William~Christopher Hayes.
\newblock \emph{The scepter of Egypt: a background for the study of the
  Egyptian antiquities in the Metropolitan Museum of Art}, volume~1.
\newblock Metropolitan Museum of Art, 1990.

\bibitem[He et~al.(2016)He, Zhang, Ren, and Sun]{he2016identity}
Kaiming He, Xiangyu Zhang, Shaoqing Ren, and Jian Sun.
\newblock Identity mappings in deep residual networks.
\newblock In \emph{European conference on computer vision}, pages 630--645.
  Springer, 2016.

\bibitem[He et~al.(2017)He, Gkioxari, Doll{\'{a}}r, and Girshick]{mrcnn}
Kaiming He, Georgia Gkioxari, Piotr Doll{\'{a}}r, and Ross~B. Girshick.
\newblock Mask {R-CNN}.
\newblock \emph{CoRR}, abs/1703.06870, 2017.
\newblock URL \url{http://arxiv.org/abs/1703.06870}.

\bibitem[Hellerstein and Stonebraker(1993)]{hellerstein1993predicate}
Joseph~M Hellerstein and Michael Stonebraker.
\newblock Predicate migration: Optimizing queries with expensive predicates.
\newblock In \emph{Proceedings of the 1993 ACM SIGMOD international conference
  on Management of data}, pages 267--276, 1993.

\bibitem[Heusel et~al.(2017)Heusel, Ramsauer, Unterthiner, Nessler, and
  Hochreiter]{heusel2017gans}
Martin Heusel, Hubert Ramsauer, Thomas Unterthiner, Bernhard Nessler, and Sepp
  Hochreiter.
\newblock Gans trained by a two time-scale update rule converge to a local nash
  equilibrium.
\newblock In \emph{Advances in Neural Information Processing Systems}, pages
  6626--6637, 2017.

\bibitem[Huang and Belongie(2017)]{huang2017arbitrary}
Xun Huang and Serge Belongie.
\newblock Arbitrary style transfer in real-time with adaptive instance
  normalization.
\newblock In \emph{Proceedings of the IEEE International Conference on Computer
  Vision}, pages 1501--1510, 2017.

\bibitem[Huh et~al.(2016)Huh, Agrawal, and Efros]{huh2016makes}
Minyoung Huh, Pulkit Agrawal, and Alexei~A Efros.
\newblock What makes imagenet good for transfer learning?
\newblock \emph{arXiv preprint arXiv:1608.08614}, 2016.

\bibitem[Iandola et~al.(2014)Iandola, Moskewicz, Karayev, Girshick, Darrell,
  and Keutzer]{iandola2014densenet}
Forrest Iandola, Matt Moskewicz, Sergey Karayev, Ross Girshick, Trevor Darrell,
  and Kurt Keutzer.
\newblock Densenet: Implementing efficient convnet descriptor pyramids.
\newblock \emph{arXiv preprint arXiv:1404.1869}, 2014.

\bibitem[Iandola et~al.(2016)Iandola, Han, Moskewicz, Ashraf, Dally, and
  Keutzer]{iandola2016squeezenet}
Forrest~N Iandola, Song Han, Matthew~W Moskewicz, Khalid Ashraf, William~J
  Dally, and Kurt Keutzer.
\newblock Squeezenet: Alexnet-level accuracy with 50x fewer parameters and< 0.5
  mb model size.
\newblock \emph{arXiv preprint arXiv:1602.07360}, 2016.

\bibitem[Jing et~al.(2017)Jing, Yang, Feng, Ye, and Song]{style-review}
Yongcheng Jing, Yezhou Yang, Zunlei Feng, Jingwen Ye, and Mingli Song.
\newblock Neural style transfer: {A} review.
\newblock \emph{CoRR}, abs/1705.04058, 2017.
\newblock URL \url{http://arxiv.org/abs/1705.04058}.

\bibitem[Jing et~al.(2019)Jing, Yang, Feng, Ye, Yu, and Song]{Jing_2019}
Yongcheng Jing, Yezhou Yang, Zunlei Feng, Jingwen Ye, Yizhou Yu, and Mingli
  Song.
\newblock Neural style transfer: A review.
\newblock \emph{IEEE Transactions on Visualization and Computer Graphics}, page
  1–1, 2019.
\newblock ISSN 2160-9306.
\newblock \doi{10.1109/tvcg.2019.2921336}.
\newblock URL \url{http://dx.doi.org/10.1109/tvcg.2019.2921336}.

\bibitem[Johnson et~al.(2019)Johnson, Douze, and J{\'e}gou]{johnson2019billion}
Jeff Johnson, Matthijs Douze, and Herv{\'e} J{\'e}gou.
\newblock Billion-scale similarity search with gpus.
\newblock \emph{IEEE Transactions on Big Data}, 2019.

\bibitem[Karras et~al.(2017)Karras, Aila, Laine, and
  Lehtinen]{karras2017progressive}
Tero Karras, Timo Aila, Samuli Laine, and Jaakko Lehtinen.
\newblock Progressive growing of gans for improved quality, stability, and
  variation, 2017.

\bibitem[Knuth(1997)]{knuth1997art}
Donald~Ervin Knuth.
\newblock \emph{The art of computer programming}, volume~3.
\newblock Pearson Education, 1997.

\bibitem[Koch et~al.(2015)Koch, Zemel, and Salakhutdinov]{koch2015siamese}
Gregory Koch, Richard Zemel, and Ruslan Salakhutdinov.
\newblock Siamese neural networks for one-shot image recognition.
\newblock In \emph{ICML deep learning workshop}, volume~2, 2015.

\bibitem[Le~Corbeiller(1974)]{le1974china}
Clare Le~Corbeiller.
\newblock \emph{China Trade Porcelain: Patterns of Exchange: Additions to the
  Helena Woolworth McCann Collection in the Metropolitan Museum of Art}.
\newblock Metropolitan Museum of Art, 1974.

\bibitem[Levy et~al.(1994)Levy, Mumick, and Sagiv]{levy1994query}
Alon~Y Levy, Inderpal~Singh Mumick, and Yehoshua Sagiv.
\newblock Query optimization by predicate move-around.
\newblock In \emph{VLDB}, pages 96--107, 1994.

\bibitem[Liao et~al.(2017)Liao, Yao, Yuan, Hua, and Kang]{liao2017visual}
Jing Liao, Yuan Yao, Lu~Yuan, Gang Hua, and Sing~Bing Kang.
\newblock Visual attribute transfer through deep image analogy.
\newblock \emph{arXiv preprint arXiv:1705.01088}, 2017.

\bibitem[Liao et~al.(2018)Liao, He, Zhao, Ngo, and Chua]{liao2018interpretable}
Lizi Liao, Xiangnan He, Bo~Zhao, Chong-Wah Ngo, and Tat-Seng Chua.
\newblock Interpretable multimodal retrieval for fashion products.
\newblock In \emph{Proceedings of the 26th ACM international conference on
  Multimedia}, pages 1571--1579, 2018.

\bibitem[Lin et~al.(2014)Lin, Maire, Belongie, Hays, Perona, Ramanan,
  Doll{\'a}r, and Zitnick]{lin2014microsoft}
Tsung-Yi Lin, Michael Maire, Serge Belongie, James Hays, Pietro Perona, Deva
  Ramanan, Piotr Doll{\'a}r, and C~Lawrence Zitnick.
\newblock Microsoft coco: Common objects in context.
\newblock In \emph{European conference on computer vision}, pages 740--755.
  Springer, 2014.

\bibitem[Liu et~al.(2015)Liu, Luo, Wang, and Tang]{liu2015faceattributes}
Ziwei Liu, Ping Luo, Xiaogang Wang, and Xiaoou Tang.
\newblock Deep learning face attributes in the wild.
\newblock In \emph{Proceedings of International Conference on Computer Vision
  (ICCV)}, December 2015.

\bibitem[Lu et~al.(2018)Lu, Huang, Fu, Guo, and Lin]{sketchretrieval}
Peng Lu, Gao Huang, Yanwei Fu, Guodong Guo, and Hangyu Lin.
\newblock Learning large euclidean margin for sketch-based image retrieval.
\newblock \emph{CoRR}, abs/1812.04275, 2018.
\newblock URL \url{http://arxiv.org/abs/1812.04275}.

\bibitem[Manning et~al.(2008)Manning, Raghavan, and
  Sch{\"u}tze]{manning2008introduction}
Christopher~D Manning, Prabhakar Raghavan, and Hinrich Sch{\"u}tze.
\newblock \emph{Introduction to information retrieval}.
\newblock Cambridge university press, 2008.

\bibitem[Marcel and Rodriguez(2010)]{marcel2010torchvision}
S{\'e}bastien Marcel and Yann Rodriguez.
\newblock Torchvision the machine-vision package of torch.
\newblock In \emph{Proceedings of the 18th ACM international conference on
  Multimedia}, pages 1485--1488, 2010.

\bibitem[Marchiori(2009)]{marchiori2009class}
Elena Marchiori.
\newblock Class conditional nearest neighbor for large margin instance
  selection.
\newblock \emph{IEEE Transactions on Pattern Analysis and Machine
  Intelligence}, 32\penalty0 (2):\penalty0 364--370, 2009.

\bibitem[Matsui et~al.(2018)Matsui, Hinami, and
  Satoh]{matsui2018reconfigurable}
Yusuke Matsui, Ryota Hinami, and Shin'ichi Satoh.
\newblock Reconfigurable inverted index.
\newblock In \emph{Proceedings of the 26th ACM international conference on
  Multimedia}, pages 1715--1723, 2018.

\bibitem[McCandless et~al.(2010)McCandless, Hatcher, Gospodneti{\'c}, and
  Gospodneti{\'c}]{mccandless2010lucene}
Michael McCandless, Erik Hatcher, Otis Gospodneti{\'c}, and O~Gospodneti{\'c}.
\newblock \emph{Lucene in action}, volume~2.
\newblock Manning Greenwich, 2010.

\bibitem[Mellina(2017)]{mellina_2017}
Clayton Mellina.
\newblock Introducing similarity search at flickr, Mar 2017.
\newblock URL
  \url{https://code.flickr.net/2017/03/07/introducing-similarity-search-at-flickr/}.

\bibitem[Nichol(2016)]{nichol2016painter}
K~Nichol.
\newblock Painter by numbers, wikiart, 2016.

\bibitem[Olah et~al.(2017)Olah, Mordvintsev, and Schubert]{olah2017feature}
Chris Olah, Alexander Mordvintsev, and Ludwig Schubert.
\newblock Feature visualization.
\newblock \emph{Distill}, 2017.
\newblock \doi{10.23915/distill.00007}.
\newblock https://distill.pub/2017/feature-visualization.

\bibitem[Omohundro(1989)]{omohundro1989five}
Stephen~M Omohundro.
\newblock \emph{Five balltree construction algorithms}.
\newblock International Computer Science Institute Berkeley, 1989.

\bibitem[Oppenheim et~al.(2015)Oppenheim, Arnold, Arnold, and
  Yamamoto]{oppenheim2015ancient}
Adela Oppenheim, Dorothea Arnold, Dieter Arnold, and Kei Yamamoto.
\newblock \emph{Ancient Egypt Transformed: The Middle Kingdom}.
\newblock Metropolitan Museum of Art, 2015.

\bibitem[Paszke et~al.(2017)Paszke, Gross, Chintala, Chanan, Yang, DeVito, Lin,
  Desmaison, Antiga, and Lerer]{paszke2017automatic}
Adam Paszke, Sam Gross, Soumith Chintala, Gregory Chanan, Edward Yang, Zachary
  DeVito, Zeming Lin, Alban Desmaison, Luca Antiga, and Adam Lerer.
\newblock Automatic differentiation in pytorch.
\newblock 2017.

\bibitem[Pedregosa et~al.(2011)Pedregosa, Varoquaux, Gramfort, Michel, Thirion,
  Grisel, Blondel, Prettenhofer, Weiss, Dubourg, et~al.]{pedregosa2011scikit}
Fabian Pedregosa, Ga{\"e}l Varoquaux, Alexandre Gramfort, Vincent Michel,
  Bertrand Thirion, Olivier Grisel, Mathieu Blondel, Peter Prettenhofer, Ron
  Weiss, Vincent Dubourg, et~al.
\newblock Scikit-learn: Machine learning in python.
\newblock \emph{the Journal of machine Learning research}, 12:\penalty0
  2825--2830, 2011.

\bibitem[Plummer et~al.(2018)Plummer, Kordas, Hadi~Kiapour, Zheng, Piramuthu,
  and Lazebnik]{plummer2018conditional}
Bryan~A Plummer, Paige Kordas, M~Hadi~Kiapour, Shuai Zheng, Robinson Piramuthu,
  and Svetlana Lazebnik.
\newblock Conditional image-text embedding networks.
\newblock In \emph{Proceedings of the European Conference on Computer Vision
  (ECCV)}, pages 249--264, 2018.

\bibitem[Radford et~al.(2015)Radford, Metz, and
  Chintala]{radford2015unsupervised}
Alec Radford, Luke Metz, and Soumith Chintala.
\newblock Unsupervised representation learning with deep convolutional
  generative adversarial networks.
\newblock \emph{arXiv preprint arXiv:1511.06434}, 2015.

\bibitem[Sandler et~al.(2018)Sandler, Howard, Zhu, Zhmoginov, and Chen]{mnv2}
Mark Sandler, Andrew~G. Howard, Menglong Zhu, Andrey Zhmoginov, and
  Liang{-}Chieh Chen.
\newblock Inverted residuals and linear bottlenecks: Mobile networks for
  classification, detection and segmentation.
\newblock \emph{CoRR}, abs/1801.04381, 2018.
\newblock URL \url{http://arxiv.org/abs/1801.04381}.

\bibitem[Schroff et~al.(2015)Schroff, Kalenichenko, and
  Philbin]{Schroff_2015_CVPR}
Florian Schroff, Dmitry Kalenichenko, and James Philbin.
\newblock Facenet: A unified embedding for face recognition and clustering.
\newblock In \emph{The IEEE Conference on Computer Vision and Pattern
  Recognition (CVPR)}, June 2015.

\bibitem[Thomaz and Giraldi(2010)]{thomaz2010new}
Carlos~Eduardo Thomaz and Gilson~Antonio Giraldi.
\newblock A new ranking method for principal components analysis and its
  application to face image analysis.
\newblock \emph{Image and vision computing}, 28\penalty0 (6):\penalty0
  902--913, 2010.

\bibitem[{Traina} et~al.(2004){Traina}, {Trains}, and {de
  Figuciredo}]{conditional-retrieval}
C.~{Traina}, A.~J.~M. {Trains}, and J.~M. {de Figuciredo}.
\newblock Including conditional operators in content-based image retrieval in
  large sets of medical exams.
\newblock In \emph{Proceedings. 17th IEEE Symposium on Computer-Based Medical
  Systems}, pages 85--90, 2004.

\bibitem[Veit et~al.(2016)Veit, Belongie, and Karaletsos]{veit2016conditional}
Andreas Veit, Serge Belongie, and Theofanis Karaletsos.
\newblock Conditional similarity networks, 2016.

\bibitem[Volker(1954)]{volker1954porcelain}
Tijs Volker.
\newblock \emph{Porcelain and the Dutch East India Company: as recorded in the
  Dagh-Registers of Batavia Castle, those of Hirado and Deshima and other
  contemporary papers; 1602-1682}, volume~11.
\newblock Brill Archive, 1954.

\bibitem[Walt et~al.(2011)Walt, Colbert, and Varoquaux]{walt2011numpy}
St{\'e}fan van~der Walt, S~Chris Colbert, and Gael Varoquaux.
\newblock The numpy array: a structure for efficient numerical computation.
\newblock \emph{Computing in Science \& Engineering}, 13\penalty0 (2):\penalty0
  22--30, 2011.

\bibitem[Wang et~al.(2014)Wang, Shen, Song, and Ji]{hashing-summary}
Jingdong Wang, Heng~Tao Shen, Jingkuan Song, and Jianqiu Ji.
\newblock Hashing for similarity search: {A} survey.
\newblock \emph{CoRR}, abs/1408.2927, 2014.
\newblock URL \url{http://arxiv.org/abs/1408.2927}.

\bibitem[Wang(2011)]{wang2011fast}
Xueyi Wang.
\newblock A fast exact k-nearest neighbors algorithm for high dimensional
  search using k-means clustering and triangle inequality.
\newblock In \emph{The 2011 International Joint Conference on Neural Networks},
  pages 1293--1299. IEEE, 2011.

\bibitem[Werner(1922)]{werner1922myths}
Edward~TC Werner.
\newblock Myths and legends of china. london: George g.
\newblock \emph{Harrap. Disertasi}, 1922.

\bibitem[Xie et~al.(2016)Xie, Girshick, Doll{\'{a}}r, Tu, and He]{resnext}
Saining Xie, Ross~B. Girshick, Piotr Doll{\'{a}}r, Zhuowen Tu, and Kaiming He.
\newblock Aggregated residual transformations for deep neural networks.
\newblock \emph{CoRR}, abs/1611.05431, 2016.
\newblock URL \url{http://arxiv.org/abs/1611.05431}.

\bibitem[Yamins et~al.(2014)Yamins, Hong, Cadieu, Solomon, Seibert, and
  DiCarlo]{yamins2014performance}
Daniel~LK Yamins, Ha~Hong, Charles~F Cadieu, Ethan~A Solomon, Darren Seibert,
  and James~J DiCarlo.
\newblock Performance-optimized hierarchical models predict neural responses in
  higher visual cortex.
\newblock \emph{Proceedings of the National Academy of Sciences}, 111\penalty0
  (23):\penalty0 8619--8624, 2014.

\bibitem[Yan et~al.(2019)Yan, Wang, Wang, Wang, and Li]{yan2019k}
Donghui Yan, Yingjie Wang, Jin Wang, Honggang Wang, and Zhenpeng Li.
\newblock K-nearest neighbors search by random projection forests.
\newblock \emph{IEEE Transactions on Big Data}, 2019.

\bibitem[Zhang et~al.(2016)Zhang, Isola, and Efros]{zhang2016colorful}
Richard Zhang, Phillip Isola, and Alexei~A Efros.
\newblock Colorful image colorization.
\newblock In \emph{European conference on computer vision}, pages 649--666.
  Springer, 2016.

\end{thebibliography}

\newpage

\appendix

\section{Visualizing Failure Cases}
\label{sec:low-dim-failure}

Figure \ref{fig:failure} (a) shows how conditioners that do not share a common support can yield low diversity conditional neighbors. Though sharing a common support is certainly helpful, it is not mandatory as shown by Figure \ref{fig:failure} (b). Some potential mitigations for these effects could be to fine tune learned embeddings to promote diverse queries, or to re-weight query outputs based on diversity. Additionally, an initial alignment with an optimal transport method could mitigate these effects \citep{grave2018unsupervised}.

\begin{figure}[h]
\centering
\includegraphics[width=.7\linewidth]{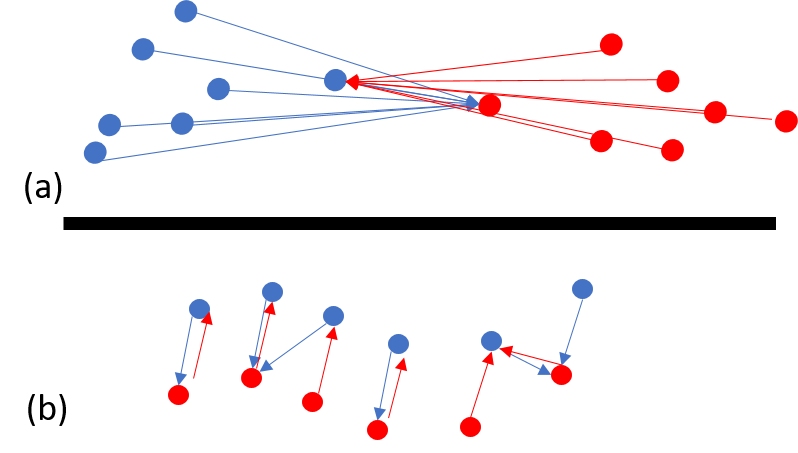}
  \caption{A schematic illustration of how conditional KNN can yield to a lack of diversity in particular geometries. (a) shows how low diversity can occur when there is no overlap of supports. Figure (b) shows how support intersection is not necessary for quality alignment}
  \label{fig:failure}
\end{figure}

\section{Additional Matches}

In addition to the matches displayed in Figure \ref{fig:feature-graphic} we provide several additional results. Figure \ref{fig:additional-match1} shows additional matches for a single query, and Figure \ref{fig:additional-match2} shows matches across several different queries. Figure \ref{fig:random-matches} shows random matches to give a sense of the method's average-case results.

\label{sec:additional-matches}
\begin{figure*}[h]
    \centering
    \includegraphics[width=\textwidth]{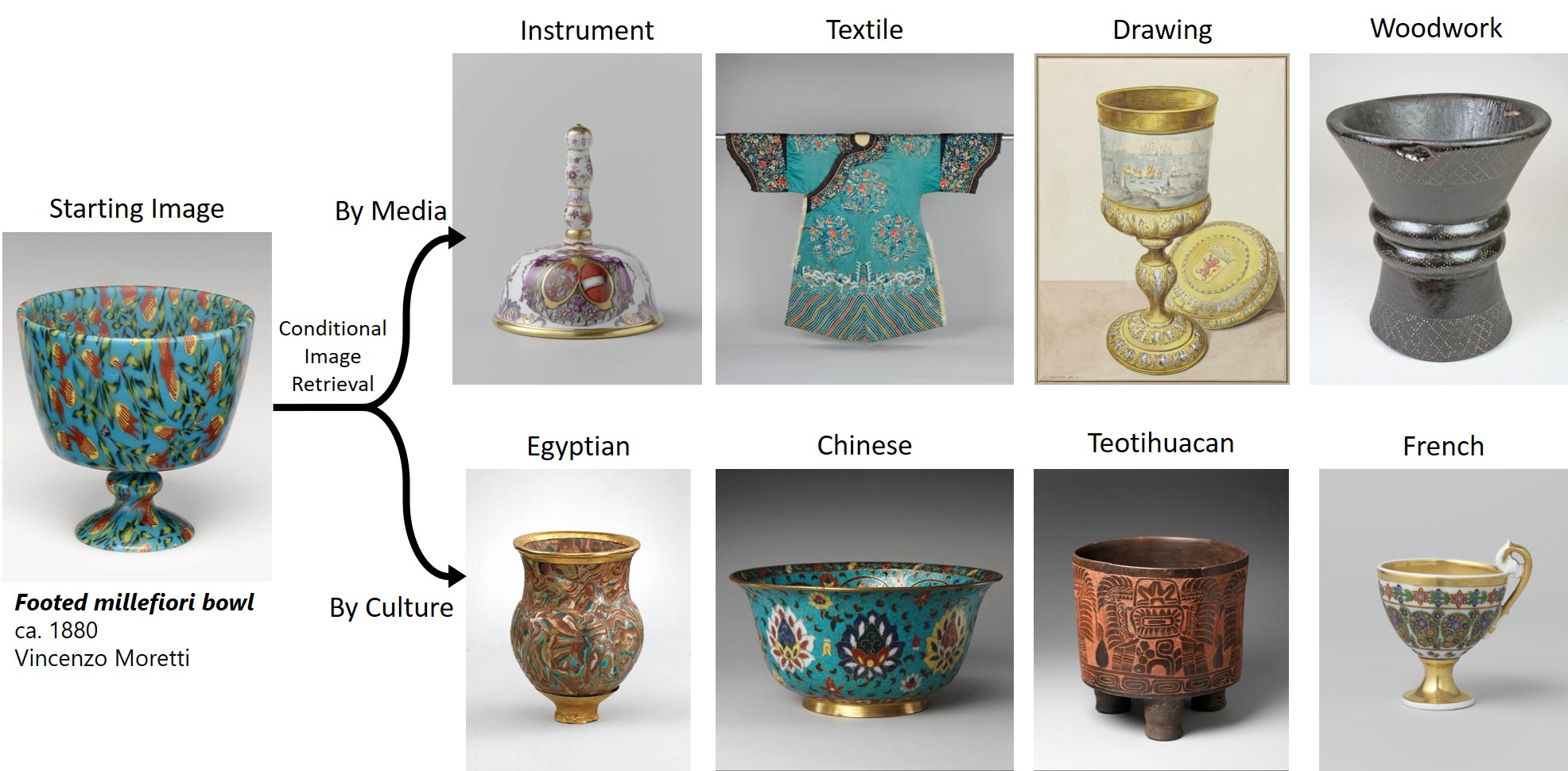} 
    \caption{Additional conditional image retrieval results on artworks from the Metropolitan Museum of Art and Rijksmuseum using media (top row text) and culture (bottom row text) as conditioners.}
    \label{fig:additional-match1}
\end{figure*}

\begin{figure*}[h]
    \centering
    \includegraphics[width=\textwidth]{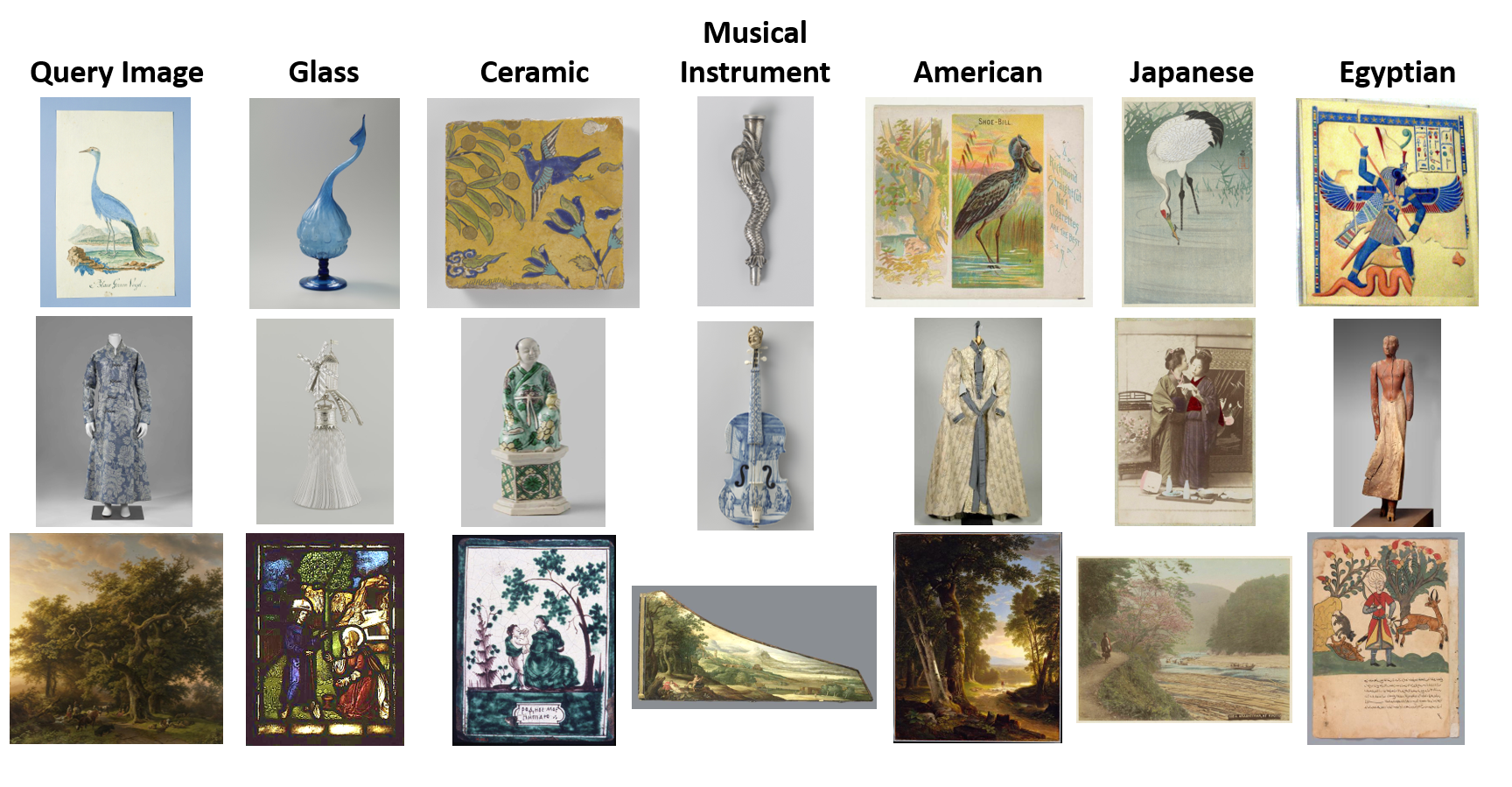} 
    \caption{Additional conditional image retrieval results on artworks from the Metropolitan Museum of Art and Rijksmuseum using top row text as conditioners.}
    \label{fig:additional-match2}
\end{figure*}

\begin{figure*}[h]
    \centering
    \includegraphics[width=\textwidth]{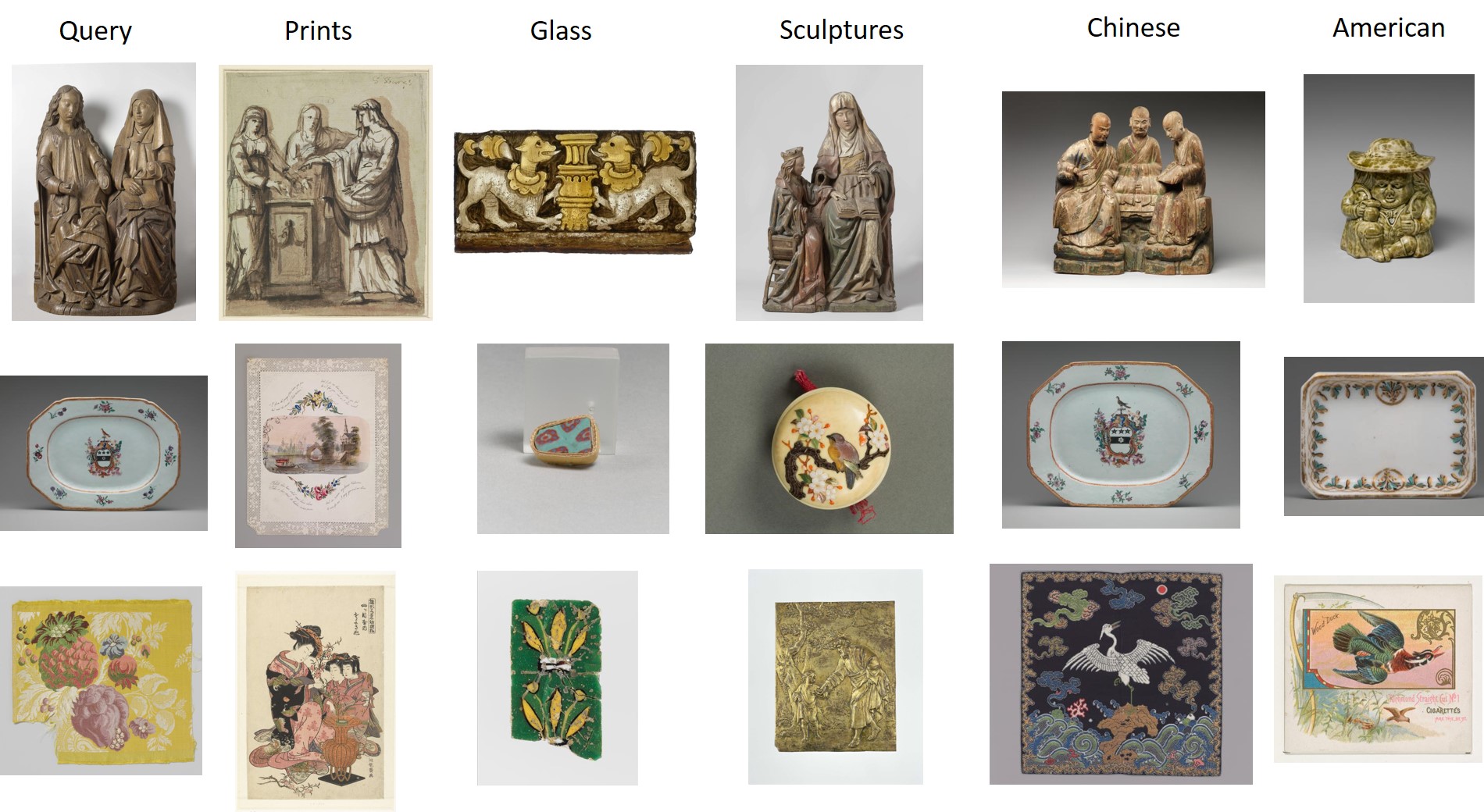} 
    \caption{Randomly selected conditional image retrieval results on artworks from the Metropolitan Museum of Art and Rijksmuseum using top row text as conditioners.}
    \label{fig:random-matches}

\end{figure*}

\newpage

\section{Proof of Theorem 1}

\label{sec:proof}

In the following analysis, suppose an \textsc{RPTree-Max} is built using a dataset $\mathcal{X} \subset \mathbb{R}^D$, of diameter $W$, with doubling dimension $\leq d$. Furthermore, assume that the size reduction rate at any given level of the tree is bounded above by $\gamma$.

\begin{lemma}
\label{lemma:1}
For any ball, $B$ of radius $R>0$ and any $0 < \epsilon < 1$, there exists a constant $c_1 > 0$ such that with probability $>1-\epsilon$, $B$ will be completely inscribed inside of an \textsc{RPTree-Max} cell of radius no more than $c_1 R d \sqrt{d} \log(d)$
\end{lemma}

\begin{proof}
We modify the proof of Theorem 12 from \citep{dhesi2010random}. In particular we let $\Delta^* = \frac{1}{\epsilon}c_5Rd\sqrt{d}\log(d)$, where $c_5$ refers to the constant of Lemma 11 of \citep{dhesi2010random} The rest of the proof proceeds without modification. 
\end{proof}

\begin{lemma}
\label{lemma:2}
For any finite set of balls, $\{B_i\}$, with constant radii $R>0$, and any $0 < \epsilon < 1$, there exists a constant $c_2 > 0$ such that with probability $>1-\epsilon$, every $B_i$ will be completely inscribed inside of an \textsc{RPTree-Max} cell of radius no more than $c_2 R d \sqrt{d} \log(d)$
\end{lemma}

\begin{proof}
We proceed by induction on the number of balls. Lemma \ref{lemma:1} provides the base case of $|\{B_i\}| = 1$. For the inductive case we assume the lemma holds for a set $\{B_i\}$ of size $n$, with $\epsilon' = \frac{\epsilon}{8}$ and constant $c'_2$. Given an additional $B_{n+1}$, we can leverage our base case to select an $\epsilon'' = \frac{\epsilon}{8}$ and constant $c''_2$. We can see that the probability that both events occur simultaneously is bounded above by:
$$ (1 - \epsilon')(1 - \epsilon'') = (1 - \frac{\epsilon}{8})(1 - \frac{\epsilon}{8}) = 1 - \frac{\epsilon}{4} -  \frac{\epsilon^2}{64} < 1 - \epsilon$$
Finally, using the new constant, $c_2 = \max(c'_2, c''_2)$, the radii criterion holds for all balls. 
\end{proof}

\begin{theorem} (Restatement of Theorem 1)  Suppose an \textsc{RPTree-Max}, $\mathcal{T}$, is built using a dataset $\mathcal{X} \subset \mathbb{R}^D$, of diameter $W$, with doubling dimension $\leq d$. Further suppose $\mathcal{T}$ is balanced with a cell-size reduction rate bounded above by $\gamma$. Let $\mathcal{S} \subseteq \mathcal{X}$ be a subset of the dataset used to build the tree and $\mathcal{B}$ a finite set of radius $R>0$ balls that cover $\mathcal{S}$. For every $0< \epsilon < 1$ there exists a constant, $c > 0$, such that with probability $> 1 - \epsilon$ the fraction of cells that contain points within $\mathcal{S}$ is bounded above by $|\mathcal{B}|2^{-log_\gamma(W/R')}$ where $R' = c R d \sqrt{d} \log(d)$
\end{theorem}

\begin{proof}
We begin by invoking Lemma \ref{lemma:2}, which shows that each ball of our covering will end up completely inscribed within small radii cells of $\mathcal{T}$. For each ball we upper bound their contribution to the total fraction of cells that contain points within $\mathcal{S}$. 

Consider any ball $B_i \in \mathcal{B}$ in the covering. By Lemma \ref{lemma:2} we know this ball is inscribed within a cell of radius $R'$. Our goal is to show that this cell must be several levels down in the tree. By our regularity conditions we know that at each subsequent level of a tree, the cell size decreases by at most a factor of $\gamma$. So to achieve the reduction in size from $W$ to $R'$, the cell must lie at or below level $\log_\gamma(W/R')$. At worst, every child of our cell contains a point within $\mathcal{S}$. Because $\mathcal{T}$ is balanced, the ratio of cell children to total cells of the tree is at most $2^{-log_\gamma(W/R')}$. At worst each ball of the cover, $\mathcal{B}$, is in a separate branch of the tree so combining these contributions yields $|\mathcal{B}|2^{-log_\gamma(W/R')}$.  

\end{proof}

\end{document}